\documentclass[twoside]{article}

\usepackage[accepted]{aistats2026}

\usepackage[round]{natbib}

\usepackage{amsmath,amsfonts,amssymb,amsthm,bm}
\usepackage{mathtools}
\usepackage{tikz}
\usetikzlibrary{arrows.meta, positioning}

\def\eqref#1{equation~\ref{#1}}

\def\1{\bm{1}}

\DeclareMathAlphabet{\mathsfit}{\encodingdefault}{\sfdefault}{m}{sl}
\SetMathAlphabet{\mathsfit}{bold}{\encodingdefault}{\sfdefault}{bx}{n}

\def\gE{{\mathcal{E}}}

\def\gG{{\mathcal{G}}}
\def\gH{{\mathcal{H}}}
\def\gI{{\mathcal{I}}}

\def\gL{{\mathcal{L}}}

\def\gP{{\mathcal{P}}}

\def\gV{{\mathcal{V}}}

\def\gZ{{\mathcal{Z}}}

\def\sP{{\mathbb{P}}}

\newcommand{\E}{\mathbb{E}}

\newcommand{\R}{\mathbb{R}}

\newtheorem{theorem}{Theorem}
\newtheorem{prop}[theorem]{Proposition}
\newtheorem{lemma}[theorem]{Lemma}

\usepackage[utf8]{inputenc} %
\usepackage[T1]{fontenc}    %
\usepackage{hyperref}       %
\usepackage{url}            %
\usepackage{booktabs}       %
\usepackage{amsfonts}       %
\usepackage{nicefrac}       %
\usepackage{microtype}      %
\usepackage{xcolor}         %
\usepackage{wrapfig}
\usepackage{graphicx}
\usepackage{natbib}
\usepackage{subcaption}
\usepackage{enumitem}
\usepackage{algorithm}
\usepackage{algorithmic}

\newcommand{\abs}[1]{\left|#1\right|}
\newcommand{\midd}{\,|\,}

\begin{document}

\runningtitle{Efficient Inference for Coupled Hidden Markov Models in Continuous Time and Discrete Space}

\twocolumn[

\aistatstitle{Efficient Inference for Coupled Hidden Markov Models \\ in Continuous Time and Discrete Space}

\aistatsauthor{ Giosue Migliorini \And Padhraic Smyth }

\aistatsaddress{ Department of Statistics \\ University of California, Irvine \And Department of Computer Science \\ University of California, Irvine } ]

\begin{abstract}
  Systems of interacting continuous-time Markov chains are a powerful model class, but inference is typically intractable in high-dimensional settings. Auxiliary information, such as noisy observations, is typically only available at discrete times, and incorporating it via a Doob's $h$-transform gives rise to an intractable posterior process that requires approximation. We introduce Latent Interacting Particle Systems, a model class parameterizing the generator of each Markov chain in the system. 
Our inference method involves estimating look-ahead functions (twist potentials) that anticipate future information, for which we introduce an efficient parameterization. We incorporate this approximation in a twisted Sequential Monte Carlo sampling scheme. We demonstrate the effectiveness of our approach on a challenging posterior inference task for a latent SIRS model on a graph, and on a neural model for wildfire spread dynamics trained on real data.
\end{abstract}

\section{INTRODUCTION}
Many real-world phenomena, from epidemics to wildfires, can be modeled as systems of interacting components evolving in continuous time, where the underlying dynamics are governed by discrete latent states \citep{lanchier2024stochastic}. This modeling approach builds upon concepts from continuous-time hidden Markov models \citep{baum1966statistical, kouemou2011history} and extends them to spatially-structured, high-dimensional processes. Interacting particle systems (IPSs) \citep{liggett1985interacting, lanchier2024stochastic} provide a powerful mathematical framework for describing local propagation dynamics in discrete state spaces and continuous time, and are an important subset of the broader class of continuous-time Markov chains (CTMCs).
We formulate our goal as performing probabilistic inference on systems whose latent dynamics follow an IPS, given only incomplete or indirect information. 

We aim to endow the IPS with a flexible parameterization (e.g., a neural network), yielding a discrete analog of latent/neural stochastic differential equations \citep{movellan2002monte,tzen2019neural,liu2020does,li2020scalable,bartosh2025sde}. While neural-based inference methods for continuous-time discrete-state processes have been explored \citep{seifner2023neural,berghaus2024neurips}, they have primarily demonstrated efficacy in low-dimensional settings.

\begin{figure*}[ht]
    \centering
\includegraphics[width=.88\linewidth]{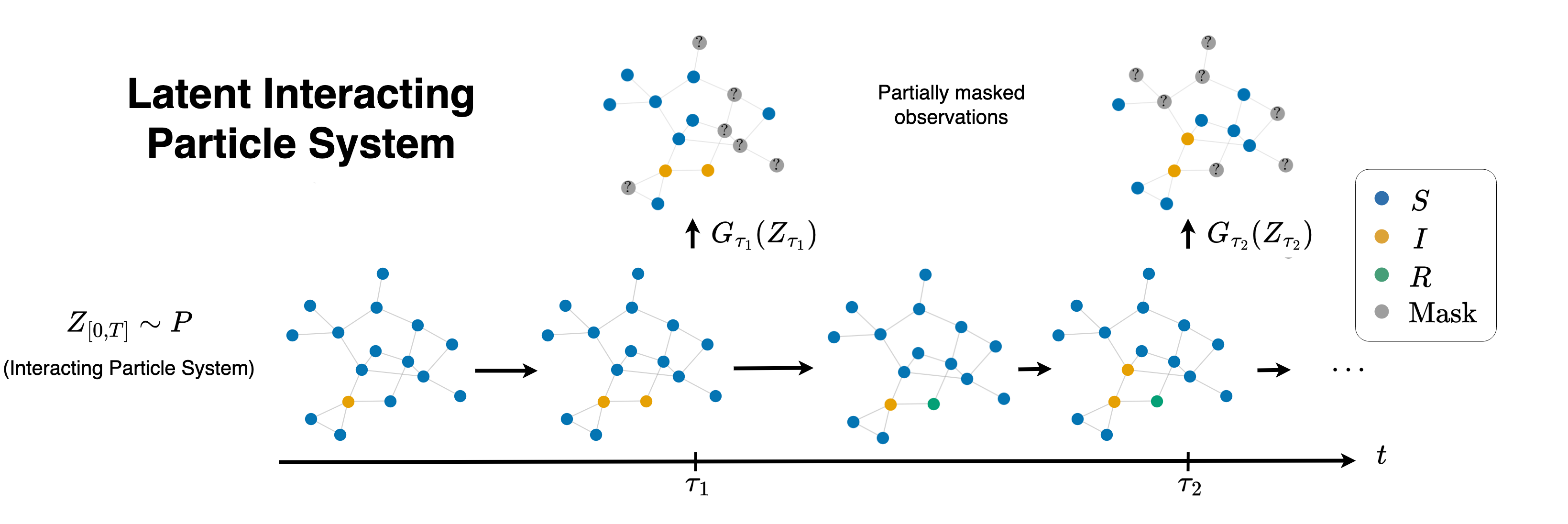}
    \caption{Example of a latent IPS as a state space model. Here, the latent trajectory is sampled from a continuous-time SIR model, and potentials are emission distributions of partially masked observations.}
    \label{fig:illustration}
\end{figure*}
The core technical challenge for inference with latent IPSs lies in sampling from a smoothed path measure over latent trajectories. To address this problem, we make the following contributions:
\begin{itemize}
    \item We propose a twisted sequential Monte Carlo (tSMC) scheme \citep{guarniero2017iterated, heng2020controlled} tailored to latent IPSs. We learn a twist function that can approximate the likelihood of future observations, and directly incorporate it into the rate matrix of an approximate posterior process, lifting the need to learn a separate proposal distribution \citep{lawson2018twisted,lawson2022sixo,lawson2023x}.
    \item We learn the twist function using a mass-covering Kullback-Leibler divergence loss in an amortized fashion \citep{zhao2024probabilistic}. Moreover, we design an efficient parameterization with favorable inductive biases. 
    \item We demonstrate this on (i) a spatial susceptible--infected--recovered--susceptible (SIRS) model on graphs with up to 256 nodes, and (ii) a neural wildfire-spread model on $64{\times}64$ grids using the WildFireSpreadTS dataset \citep{gerard2023wildfirespreadts}.
\end{itemize}

\section{BACKGROUND}

\subsection{Continuous-time Markov chains}
Consider a Markov process $Z_{[0,T]}\coloneq\{Z_t\}_{t\in[0,T]}$ taking values in a discrete state space $\gZ=\{1,\dots,V\}$ over a finite time horizon $[0,T]$. 
This system, known as a CTMC, is characterized by an initial distribution $p_0\in\gP(\gZ)$ and a measurable family of rate matrices 
$[R_t(z,\tilde z)]_{z,\tilde z\in \gZ}, \,t\in [0,T]$, where each entry satisfies
\begin{equation*}\label{eq:rates}
    R_t(z,\tilde z) \coloneq \lim_{\Delta t \to 0} \frac{1}{\Delta t}\,\sP(Z_{t+\Delta t}=\tilde z \mid Z_t=z),
\end{equation*}
for $\tilde z\neq z$, and $R_t(z,z) \coloneq -\sum_{\tilde z\neq z}R_t(z,\tilde z)$ on the diagonal.
Assuming $\sup_{t,\,z} -R_t(z,z)<\infty$ (non-explosion), sample paths are c\`adl\`ag, piecewise-constant, with finitely many jumps on $[0,T]$.

We refer to the induced distribution on the Skorokhod space $D([0,T],\gZ)$ as the \emph{path measure}. 
Writing $0<t_1<\cdots<t_N < T$ for the jump times and $z_{t}$ for the state at a time $t\in[0,T]$, the density of this measure with respect to counting measure on states and Lebesgue measure on jump times is
\begin{equation*}
    p_0(z_0)\;
    \Bigg[\prod_{n=1}^N R_{t_n}(z_{t_{n-1}},z_{t_n})\Bigg]\,
    \exp\!\left(\int_0^T R_t(z_t,z_t)\,dt\right).
\end{equation*}
A detailed introduction can be found in \cite{norris,del2017stochastic}. We provide an overview of inference with CTMCs in Appendix \ref{appdx:background}.

\subsection{Interacting particle systems}
While CTMCs can be extended to high-dimensional systems, the size of the associated rate matrix increases exponentially with the number of dimensions, making exact inference intractable. 
We restrict attention to state spaces of the form $\gZ = \gV^d$, where $\gV = \{1,\dots,V\}$ is a fixed \textit{vocabulary}. 
We assume that transitions affect only a single coordinate (i.e., dimension) at a time, and that the rate for updating a coordinate $i$ may depend on the current global state $z\in\gZ$. 
The dependence of the rate on $z$ is often specified through a graph $\gG=(\gI,\gE)$, where $\gI=\{1,\dots,d\}$ indexes coordinates and $\gE$ encodes neighborhood structure (e.g., spatial dependence) \citep{liggett1985interacting}. 

To describe local updates, we use the shorthand 
\[
z^{i\to v} \;=\; (z^1,\dots,z^{i-1},v,z^{i+1},\dots,z^d)
\]
for the configuration obtained from $z$ by replacing the $i$-th coordinate with $v\in\gV$. 
The local rate $r_{i,t}(v\mid z)\ge 0$ is defined as the instantaneous intensity of transitioning from $z$ to $z^{i\to v}$, i.e.
\[
R_t(z,z^{i\to v}) \;=\; r_{i,t}(v\mid z), \qquad v\neq z^i.
\]
Under these assumptions, the global generator $R_t$ decomposes as a sum of local generators:
\begin{equation}
    R_t(z,\tilde z) = \sum_{i\in \gI} r_{i,t}(\tilde z^i \midd z) \prod_{j\neq i}\delta_{z^j}(\tilde z^j),
\end{equation}
reflecting the assumption that only one coordinate may change at a time. Typically the local rates $r_{i,t}(v\mid z)$ depend only on the neighborhood $(z^j)_{j\in\mathcal N(i)}$ of $i$ in the graph $\gG$. CTMCs with generators of this form are known as \emph{interacting particle systems} (IPS) on a finite state space and finite graph \citep{liggett1985interacting, lanchier2017stochastic, lanchier2024stochastic}. We are interested in system that may exhibit time dependence, i.e. non-homogeneous processes \citep{norris}.

\textbf{Example \textit{(SIR model)}:} In an SIR model on a network of $d$ nodes, the state space is $\{S,I,R\}^d$ and the rate of infection of each node can be a function of the states of its neighbors. For an illustration, see the process $Z_{[0,T]}$ in Figure \ref{fig:illustration}.

\paragraph{Simulation.}\label{sec:simulation} Exact simulation from IPSs is feasible for models like the contact process and the voting process \citep{lanchier2017stochastic}, leveraging a construction from independent Poisson processes \citep{lanchier2017stochastic}. However, simulation can be challenging for non-homogeneous processes. For this reason, in this work we employ a simple first-order Euler discretization \citep{sun2023score}. We assume conditional independence among coordinates in intervals of width $\Delta_t$, and let the transition probability be:
\begin{equation}\label{eq:prod-kernel}
 q_{t,\Delta_t}(\tilde z\mid z)
\;:=\;
\prod_{i\in\gI}\!\left(
\delta_{\tilde z^i,z^i}
\;+\;
\Delta_t  r_{i,t}(\tilde z^i\mid z)
\right),
\end{equation}
where $r_{i,t}(z^i\mid z):=-\sum_{v\neq z^i} r_{i,t}(v\mid z)$.
Note that $ q_{t,\Delta_t}$ may assign mass to multi-site flips for any finite step size. Despite its simplicity, this simulation scheme has shown good performance in generative modeling applications, such as simulating the reverse process of discrete diffusion and discrete flow matching models \citep{campbellgenerative, gat2024discrete}. 

In the following Proposition, we show that the error of this kernel can be bounded in total variation. Let $
\lambda_t(z)\;:=\;\sum_{i\in\gI}\sum_{v\neq z^i} r_{i,t}(v\mid z)$.
We assume that the following properties hold for each $t\in[0,T]$:

\begin{enumerate}[label=(A\arabic*),leftmargin=*, align=left]
\item \textbf{Bounded total rate:} $\sup_{t,z}\lambda_t(z)\le \bar\lambda<\infty$. 
\item \textbf{Small time interval:} we bound the step size as $\Delta_t \le 1/\sup_{t,z,i}\sum_{v\neq z^i} r_{i,t}(v\midd z)$.
\item \textbf{Local Lipschitz in time:} for all $u\in[0,\Delta_t]$,
$\big|r_{i,t+u}(v\mid z)-r_{i,t}(v\mid z)\big|\le L\,u$. 
\end{enumerate}
Since $\gZ$ is finite, we identify the transition kernel of an IPS $P_{t,\Delta_t}(z_t,dz)\coloneq \mathbb{P}(Z_{t+\Delta_t}\in d z\mid Z_t=z_t)$ with its pmf
$p_{t,\Delta_t}(\tilde z\mid z)\coloneq P_{t,\Delta_t}(z,\{\tilde z\})$.
\begin{prop}\label{prop:tv-compact}
Under \emph{(A1)–(A3)}, there exists $C<\infty$ such that
\begin{equation*}
    \big\|q_{t, \Delta_t}(\,\cdot\midd z) - p_{t,\Delta_t}(\,\cdot\midd z)\big\|_{\mathrm{TV}}
\;\le\; C\,\Delta_t^2,
\end{equation*}
uniformly in $t\in[0,T]$ and $z\in\gZ$.
\end{prop}
A proof is provided in Appendix~\ref{proof:tv-compact}.

\subsection{Sequential Monte Carlo}
In our approximate posterior inference scheme for IPSs, we employ sequential Monte Carlo (SMC) methods. SMC approximates a terminal target $\pi_T$ by evolving a population of weighted samples through a sequence of intermediate targets $\{\pi_t\}_{t\in[0,T]}$ \citep{doucet2009tutorial,naesseth2019elements,chopin_introduction_2020}. We assume $\pi_0$ is easy to sample from, and that we can evaluate unnormalized densities $\{\gamma_t\}_{t\in[0,T]}$ with $\pi_t \propto \gamma_t$. In our setting, it is natural to let $\pi_t$ be a distribution on path prefixes $Z_{[0,t]}\in D([0,t],\gZ)$. Given particles $\{Z^{(s)}_{[0,t-\Delta_t]},\,w^{(s)}_{t-\Delta_t}\}_{s=1}^S$ targeting $\pi_{t-\Delta_t}$ and a proposal kernel
$q_t(\cdot\mid Z^{(s)}_{t-\Delta_t})$ that advances each trajectory by $\Delta_t$, we draw a segment
$Z^{(s)}_{(t-\Delta_t,t]}\sim q_t(\cdot\mid Z^{(s)}_{t-\Delta_t})$ and update unnormalized incremental weights as
\begin{gather}
\tilde w_t^{(s)}
\;=\;
\frac{\gamma_t\!\big(dZ^{(s)}_{[0,t]}\big)}
{\gamma_{t-\Delta_t}\!\big(dZ^{(s)}_{[0,t-\Delta_t]}\big)\,
q_t\!\big(Z^{(s)}_{(t-\Delta_t,t]}\mid Z^{(s)}_{t-\Delta_t}\big)}\,, \nonumber
\\
w_t^{(s)} \propto w^{(s)}_{t-\Delta_t}\,\tilde w_t^{(s)},
\quad
\bar w_t^{(s)}=\frac{w_t^{(s)}}{\sum_{j=1}^S w_t^{(j)}}. \label{eq:smc_weights}
\end{gather}

A distinctive feature of SMC is resampling, where samples are selected and propagated based on their importance weights. This is implemented by drawing ancestry indices according to $\bar w_t^{(1:S)}$, adaptively when the effective sample size $\mathrm{ESS}_t = \big(\sum_{s=1}^S (\bar w_t^{(s)})^2\big)^{-1}$ falls below a threshold \citep{naesseth2019elements}. Two design choices are central to performance: \emph{(i) intermediate targets} $\{\pi_t\}$, which can be crafted to reduce variance of incremental weights; and \emph{(ii) the proposal distribution} $q_t$, which should ensure proposed segments land in high-probability regions of $\pi_t$. Poor choices of either can trigger rapid weight collapse, especially in high-dimensional, sparse-observation regimes \citep{naesseth2019elements,chopin2023computational}.

\section{EFFICIENT INFERENCE FOR LATENT INTERACTING PARTICLE SYSTEMS}

\subsection{Posterior path measure for an IPS}

Let $Z_{[0,T]}$ be an Interacting Particle System (IPS) with path measure $P$, initial distribution $p_0$, and local transition rates $r_{i,t}(v \midd z)$. 
We are interested in conditioning this process on auxiliary information available at discrete times $\tau_1,\dots,\tau_K \in [0,T]$. 
We represent this information by nonnegative \emph{potential functions} $G_{\tau_k}: \mathcal{Z}\to\mathbb{R}_+$ that can be evaluated pointwise. 
The resulting \emph{posterior path measure} is
\begin{equation}\label{eq:posterior_path_measure}
    P^\star(dZ_{[0,T]}) \;\propto\; 
    \Bigg(\prod_{k=1}^K G_{\tau_k}(Z_{\tau_k})\Bigg) P(dZ_{[0,T]}),
\end{equation}
a special case of a Feynman–Kac (FK) path measure \citep{chopin_introduction_2020, lu_guidance_2024, park2025amortized}. A key property is that the process $Z^\star_{[0,T]}$ governed by $P^\star$ is itself an IPS. 

This follows from Doob’s $h$-transform applied to the original dynamics \citep{del2017stochastic,corstanje2023conditioning,corstanje2025guided}. 
Define the \emph{look-ahead function}
\begin{equation}\label{eq:lookahead}
h^\star_t(z) \coloneq \mathbb{E}_P\Big[\prod_{k:\tau_k>t} G_{\tau_k}(Z_{\tau_k}) \,\big|\, Z_t=z\Big],
\qquad h^\star_T(z)=1,
\end{equation}
which is right-continuous with left limits (càdlàg) with reset conditions at potential times \citep{eichentropic}:
\begin{equation}\label{eq:reset}
h^\star_{\tau_k^-}(z) \;=\; G_{\tau_k}(z)\,h^\star_{\tau_k}(z).
\end{equation}
Note that we can marginalize \eqref{eq:posterior_path_measure} to any time $t\in [0,T]$, and get
\begin{align}
    P^\star(dZ_{[0,t]}) \propto P(dZ_{[0,t]}) h^\star_{t}(Z_t) \prod_{k:\tau_k \leq t}G_{\tau_k}(Z_{\tau_k}).
\end{align}
The initial distribution and local rates of the posterior IPS are characterized in the following Proposition.
\begin{prop}[Doob’s $h$-transform for IPS]\label{prop:doob_h_transform}
Under $P^\star$, the process $\{Z^\star_t\}$ is an IPS with initial distribution
\[
p^\star_0(z) \;\propto\; p_0(z)\,h^\star_0(z),
\]
and off-diagonal local transition rates
\begin{equation}\label{eq:doob_rates}
r^{\star}_{i,t}(v \mid z) \;=\; r_{i,t}(v \mid z)\,s^{\star}_{i,t}(v,z)
,
\qquad v\neq z^i,
\end{equation}
where $ s^{\star}_{i,t}(v,z) \coloneq {h^\star_t(z^{i\to v})}/{h^\star_t(z)}$.
Diagonal terms are defined as usual by
$r^{\star}_{i,t}(z^i\mid z) = -\sum_{v\neq z^i} r^{\star}_{i,t}(v\mid z)$.
\end{prop}
Note that this proposition follows from standard arguments on the conditioning of CTMCs, see \citep{huang2016reconstructing, corstanje2023conditioning,corstanje2025guided}. We provide a complete proof of Proposition \ref{prop:doob_h_transform} in Appendix \ref{proof:doob_h_transform}. 
This formulation generalizes the posterior distribution used in reward-guided fine-tuning for discrete diffusion models \citep{liderivative2025, wangfine2025, lee_debiasing_2025} to multiple potential times. 
Simulating directly from these transformed rates $r^{\star}_{i,t}$ is generally intractable because computing the look-ahead function $h^\star_t(z)$ requires solving a high-dimensional integral over future trajectories under $P$. 

\subsection{Latent IPS}

Let $\{P_\theta : \theta\in\Theta\}$ denote a family of prior path measures on $D([0,T],\gZ)$, 
where each $P_\theta$ corresponds to an IPS with initial distribution $p_{0}^{\theta}$ 
and local transition rates $r^\theta_{i,t}(v\mid z)$. 

Auxiliary information about the latent trajectory is available in the form of discrete-time observations 
\[
Y = (Y_{\tau_1},\dots,Y_{\tau_K}), \qquad 0<\tau_1<\cdots<\tau_K\le T.
\] 
Here and throughout, we use $\theta$ to denote the parameters of the full generative model, including both the prior IPS dynamics and the observation model.
For an observed trajectory $y_{1:K}$, potential functions in \eqref{eq:posterior_path_measure} correspond to emission distributions
\[
G_{\tau_k,\theta}(z) \;=\; p_\theta(y_{k}\mid Z_{\tau_k}=z), \qquad k=1,\dots,K,
\]
where we omit $Y_{\tau_k}$ to simplify the notation.
We denote the corresponding posterior path measure by $P^\star_\theta$, which, by Proposition~\ref{prop:doob_h_transform}, is again an IPS. 
We refer to this model as a \emph{Latent Interacting Particle System (Latent IPS)}. Figure \ref{fig:illustration} illustrates this construction for a latent susceptible-infected-recovered (SIR) model on a graph.

\paragraph{Maximum likelihood.} The logarithm of the marginal likelihood of the observations is
\begin{equation}\label{eq:latent_ips_likelihood}
    \mathcal{L}(\theta) \coloneq \log \mathbb{E}_{P_\theta}\Bigg[ \prod_{k=1}^K G_{\tau_k,\theta}(Z_{\tau_k}) \Bigg].
\end{equation}
This objective can be optimized via gradient steps using Fisher's identity (see Appendix~\ref{proof:fisher}), rewriting $\nabla_\theta \mathcal{L}(\theta)$ as
\begin{align}\label{eq:mle}
    \E_{P^\star_\theta}\left[\sum_{k=1}^K \nabla_\theta\log G_{\tau_k,\theta}(Z_{\tau_k}) + \nabla_\theta\log P_\theta(Z_{[0,T]}) \right].
\end{align}
These updates are not directly tractable, as they require samples from $P^\star_\theta$. Viable options include optimizing an evidence lower bound \citep{hinton1995wake}, importance sampling \citep{bornschein2014reweighted}, and SMC \citep{lawson2023x, mcnamara2024sequential}. Due to the sequential nature of latent IPSs, we opt for an SMC approach.

\subsection{Twisted SMC for posterior inference in latent IPS}
Classic design choices for SMC algorithms, such as the bootstrap particle filter (BPF) \citep{doucet2009tutorial}, notoriously display poor performance in continuous-time problems with sparse potentials, as weights are uniform in between potential times, leading to particle degeneracy \citep{chopin2023computational}.

\paragraph{Twisted targets.} Twisted SMC (tSMC) 
 addresses this limitation by specifying a favorable choice for the intermediate target distributions of SMC \citep{guarniero2017iterated, heng2020controlled}. The core idea is to introduce a learnable function $h_t^\psi: \mathcal{Z} \rightarrow \mathbb{R}_+$, the \textit{twist function}, approximating the look-ahead function $h^\star_t(z)$ from \eqref{eq:lookahead}.
The tSMC algorithm targets a sequence of \textit{twisted} distributions 
\begin{align}
    P^\psi_\theta(dZ_{[0,t]})  &\propto P_{\theta}(dZ_{[0,t]}) h^\psi_{t}(Z_t)\prod_{k:\tau_k \leq t}G_{\tau_k,\theta}(Z_{\tau_k}),
\end{align}
for $t\in[0,T]$. Note that $h^\psi_t$ can be a function of future observations $y_{>t}$, and we omit this dependence for notational simplicity.

\paragraph{Twist-induced proposal.}
We can approximately sample from this IPS using an approximation of a first-order Euler discretization as in \ref{sec:simulation}, resulting in the twist-induced proposal distribution
\begin{align}\label{eq:twist_approx_kernel}
    q^{\theta,\psi}_{t,\Delta_t}(z\mid z_t)
    &=\prod_{i\in\gI} \Bigl(\delta_{z^i_{t}, z^i} + \Delta_t \, r^{\theta,\psi}_{i,t}(z^i \mid z_{t})\Bigr), \\
    r^{\theta,\psi}_{i,t}(z^i \mid z_{t}) \!&\coloneq \!
    \begin{cases}
        r^{\theta}_{i,t}(z^i \mid z_{t}) {s^\psi_{i,t}(z^i,z_t)}, &z^i\neq z^i_t, \\
        - \!\sum_{v\neq z^i_t }r^{\theta}_{i,t}(v\mid z_{t}) {s^\psi_{i,t}(v,z_t)},\!\! & z^i = z^i_t,
    \end{cases}
\end{align}
where $s^\psi_{i,t}(z^i,z_t)\coloneq {h^\psi_t(z_t^{i\rightarrow z^i})}/{h^\psi_t(z_t)}$ is the concrete score \citep{meng2022concrete}.
Note that the twist function $h^\psi$ might not satisfy the reset conditions in \eqref{eq:reset}. We can quantify the extent to which it is violated by introducing the \textit{reset residual}
\begin{equation*}
    \rho_t(z) \coloneq \left(\log h^\psi_{t^-}(z)-\log G_t(z)-\log h^\psi_t(z)\right){\boldsymbol{1}_{t\in\{\tau_k\}}} .
\end{equation*}

Let the incremental effective sample size (ESS) with respect to the target posterior be defined as
\begin{gather}\label{eq:ess}
    \text{ESS}_{t,\Delta_t}(z_t)\coloneq {\mathbb{E}_{q^{\theta,\psi}_{t,\Delta_t}(\cdot\mid z_{t})}\left[ \left(\frac{p^{\theta,\star}_{t,\Delta_t}(Z_{t+\Delta_t}\midd z_{t}) }{ q^{\theta,\psi}_{t,\Delta_t}(Z_{t+\Delta_t}\midd z_{t})} \right)^2\right]}^{-1},
\end{gather}
where 
\begin{equation*}
    p^{\theta,\star}_{t,\Delta_t}(z\midd z_{t}) \propto {p^{\theta}_{t,\Delta_t}(z\midd z_{t}) h^\star_{t+\Delta_t}(z)G_{t}(z)^{ \boldsymbol{1}_{t\in\{\tau_k\}}}}.
\end{equation*}
In the following Theorem, we show that the incremental ESS at each step
can be lower bounded by composing terms depending on the approximation error of the twist function, the reset residual, and the discretization error.

\begin{theorem}
\label{thm:ess-final}
For $t\in [0,T]$, bound the twist error and the reset residual by
\begin{align*}\varepsilon_t&\coloneq\sup_{z\in\gZ}\big|\log h^\star_t(z)-\log h^\psi_t(z)\big|, \quad \delta_t \coloneq\sup_{z\in\gZ} |\rho_t(z)|.
\end{align*}
Under (A1)-(A3) for the prior dynamics and assuming $h^\psi$ to be Lipschitz continuous between observation times, there exists a $C<\infty$ such that the incremental ESS in \eqref{eq:ess} satisfies, uniformly in $z_t$,
\[
\mathrm{ESS}_{t,\Delta_t}(z_t)\ \ge\ \frac{\exp(-4(\varepsilon_{t+\Delta_t} + \delta_{t+\Delta_t} ))}{1 +C\Delta_t^2\,}.
\]
\end{theorem}
For a proof, see Appendix~\ref{proof:ess-final}.

\usetikzlibrary{positioning, arrows.meta, shapes.geometric, matrix, calc, backgrounds, fit}
\begin{figure*}[ht]
\centering
\begin{tikzpicture}[
    >=Stealth, font=\sffamily, thick,
    line join=round, line cap=round, 
    vecbox/.style  = {draw=black!60, fill=blue!10, 
                      minimum width=1.6cm, minimum height=0.55cm,
                      thick, font=\small, align=center, rounded corners=2pt},
    operator/.style= {circle, draw=black!60, fill=green!10, 
                      inner sep=2pt, font=\small, thick},
    swaptext/.style= {fill=gray!3, align=center,
                      font=\scriptsize, inner sep=4pt, rounded corners=2pt},
    cO/.style = {fill=orange!45},
    cB/.style = {fill=blue!45!cyan!50},
    cG/.style = {fill=green!45!gray!55}
  ]
  
  \def\Ya{ 0.76} \def\Yb{ 0.38} \def\Yc{ 0.0} \def\Yd{-0.38} \def\Ye{-0.76}
  \def\Ytop{0.95} \def\Ybot{-0.95}
  \def\cellh{0.38} 
  \def\xd{0.25} \def\yd{0.18} 
 
  \node[font=\small\bfseries, above=0.15cm] at (0.22, \Ytop) {$z_t$};
 
  \def\Zw{0.45}  
  \def\Zx{0.0}   
 
  \draw[fill=white, draw=none] (\Zx, \Ybot) rectangle (\Zx+\Zw, \Ytop);
 
  \filldraw[cO, draw=none] (\Zx, \Ya-\cellh/2) rectangle ++(\Zw, \cellh);
  \filldraw[cB, draw=none] (\Zx, \Yb-\cellh/2) rectangle ++(\Zw, \cellh);
  \filldraw[cB, draw=none] (\Zx, \Yc-\cellh/2) rectangle ++(\Zw, \cellh);
  \filldraw[cG, draw=none] (\Zx, \Yd-\cellh/2) rectangle ++(\Zw, \cellh);
  \filldraw[cO, draw=none] (\Zx, \Ye-\cellh/2) rectangle ++(\Zw, \cellh);
 
  \foreach \y in {0.57, 0.19, -0.19, -0.57}
    \draw[black!60, thin] (\Zx, \y) -- (\Zx+\Zw, \y);
    
  \draw[draw=black!60, thin] (\Zx, \Ybot) rectangle (\Zx+\Zw, \Ytop);
 
  \node[font=\small, left=0.1cm] at (\Zx, \Ya) {$i=1$};
  \node[font=\small, left=0.1cm] at (\Zx, \Ye) {$i=d$};
  \node[font=\small, left=0.18cm] at (\Zx, \Yc) {$\vdots$};
 
  \def\Tx{1.8} 
  \def\Tw{0.45}  
 
  \begin{scope}[on background layer]
    \draw[fill=blue!10, draw=black!60, thin] (\Tx, \Ytop) -- ++(\xd,\yd) -- ++(3*\Tw,0) -- (\Tx+3*\Tw, \Ytop) -- cycle;
    \draw[fill=blue!20, draw=black!60, thin] (\Tx+3*\Tw, \Ytop) -- ++(\xd,\yd) -- ++(0,-1.9) -- (\Tx+3*\Tw, \Ybot) -- cycle;
    
    \draw[black!60, thin] (\Tx+1*\Tw, \Ytop) -- ++(\xd,\yd);
    \draw[black!60, thin] (\Tx+2*\Tw, \Ytop) -- ++(\xd,\yd);
    \foreach \y in {0.57, 0.19, -0.19, -0.57}
      \draw[black!60, thin] (\Tx+3*\Tw, \y) -- ++(\xd,\yd);
  \end{scope}
 
  \draw[fill=blue!5, draw=none] (\Tx, \Ybot) rectangle (\Tx+3*\Tw, \Ytop);
 
  \filldraw[cO, draw=none] (\Tx+1*\Tw, \Ya-\cellh/2) rectangle ++(\Tw, \cellh);
  \filldraw[cB, draw=none] (\Tx+0*\Tw, \Yb-\cellh/2) rectangle ++(\Tw, \cellh);
  \filldraw[cB, draw=none] (\Tx+0*\Tw, \Yc-\cellh/2) rectangle ++(\Tw, \cellh);
  \filldraw[cG, draw=none] (\Tx+2*\Tw, \Yd-\cellh/2) rectangle ++(\Tw, \cellh);
  \filldraw[cO, draw=none] (\Tx+1*\Tw, \Ye-\cellh/2) rectangle ++(\Tw, \cellh);
 
  \foreach \y in {0.57, 0.19, -0.19, -0.57}
    \draw[black!60, thin] (\Tx, \y) -- (\Tx+3*\Tw, \y);
  \draw[black!60, thin] (\Tx+1*\Tw, \Ybot) -- (\Tx+1*\Tw, \Ytop);
  \draw[black!60, thin] (\Tx+2*\Tw, \Ybot) -- (\Tx+2*\Tw, \Ytop);
  
  \draw[draw=black!60, thin] (\Tx, \Ybot) rectangle (\Tx+3*\Tw, \Ytop);
 
  \node[font=\small\bfseries, above=0.15cm] at (\Tx+1.5*\Tw+\xd/2, \Ytop+\yd) {$\Phi_t(c_t)$};
  
  \node[font=\small] at (\Tx+0.5*\Tw, \Ybot-0.2) {$1$};
  \node[font=\small] at (\Tx+1.5*\Tw, \Ybot-0.2) {$\dots$};
  \node[font=\small] at (\Tx+2.5*\Tw, \Ybot-0.2) {$V$};
 
  \draw[->, dashed, black!50, thin] (\Zx+\Zw+0.05, \Ya) -- (\Tx+1*\Tw-0.05, \Ya);
  \draw[->, dashed, black!50, thin] (\Zx+\Zw+0.05, \Yb) -- (\Tx+0*\Tw-0.05, \Yb);
  \draw[->, dashed, black!50, thin] (\Zx+\Zw+0.05, \Yc) -- (\Tx+0*\Tw-0.05, \Yc);
  \draw[->, dashed, black!50, thin] (\Zx+\Zw+0.05, \Yd) -- (\Tx+2*\Tw-0.05, \Yd);
  \draw[->, dashed, black!50, thin] (\Zx+\Zw+0.05, \Ye) -- (\Tx+1*\Tw-0.05, \Ye);
 
  \def\Ex{4.2}
  \def\Ew{0.45}          
  \def\boxH{0.38}        
 
  \newcommand{\drawBlock}[2]{
    \draw[fill=blue!10, draw=black!60, thin] (\Ex, #1+\boxH/2) -- ++(\xd,\yd) -- ++(\Ew,0) -- (\Ex+\Ew, #1+\boxH/2) -- cycle;
    \draw[fill=blue!20, draw=black!60, thin] (\Ex+\Ew, #1+\boxH/2) -- ++(\xd,\yd) -- ++(0,-\boxH) -- (\Ex+\Ew, #1-\boxH/2) -- cycle;
    \draw[#2, draw=black!60, thin] (\Ex, #1-\boxH/2) rectangle ++(\Ew, \boxH);
  }
 
  \drawBlock{\Ye}{cO}
  \drawBlock{\Yd}{cG}
  \drawBlock{\Yc}{cB}
  \drawBlock{\Yb}{cB}
  \drawBlock{\Ya}{cO}
 
  \draw[->, thick] (\Tx+3*\Tw+\xd+0.1, 0) -- (\Ex-0.1, 0);
 
  \draw[-, thick] (\Ex+\Ew+\xd+0.1, 0) -- (5.4, 0);
        
  \node[font=\normalsize] (Sum) at (5.9, 0) {$\displaystyle\sum_{i=1}^d$};
 
  \def\Px{7.9} 
  \def\Py{1.45} 
 
  \node[vecbox] (StMid) at (\Px, 0) {$S_t(z_t)$};
  
  \draw[->, thick] (6.3, 0) -- (StMid.west);
 
  \node[vecbox] (StTop) at (\Px,  \Py) {$S_t(z_t^{i \to v})$};
  \node[vecbox] (StBot) at (\Px, -\Py) {$S_t(z_t^{j \to u})$};
 
  \draw[->, thick] (StMid.north) -- node[swaptext, right=2pt] {$+ \Phi_t(c_t)_{[i, v]} - \Phi_t(c_t)_{[i, z^i_t]}$} (StTop.south);
  \draw[->, thick] (StMid.south) -- node[swaptext, right=2pt] {$+ \Phi_t(c_t)_{[j, u]} - \Phi_t(c_t)_{[j, z^j_t]}$} (StBot.north);
 
  \node[font=\small, anchor=west] (OutTop) at (\Px+2.6,  \Py) {$\log h_t^\psi(z_t^{i \to v})$};
  \node[font=\small, anchor=west] (OutMid) at (\Px+2.6,  0.0) {$\log h_t^\psi(z_t)$};
  \node[font=\small, anchor=west] (OutBot) at (\Px+2.6, -\Py) {$\log h_t^\psi(z_t^{j \to u})$};
 
  \draw[->, thick] (StTop.east) -- node[operator] {$\phi$} (OutTop.west);
  \draw[->, thick] (StMid.east) -- node[operator] {$\phi$} (OutMid.west);
  \draw[->, thick] (StBot.east) -- node[operator] {$\phi$} (OutBot.west);
 
  \begin{scope}[on background layer]
    \node[draw=black!30, dashed, rounded corners=6pt, fill=gray!3,
          fit=(StTop)(StBot)(StMid)(OutTop)(OutMid)(OutBot),
          inner sep=8pt] (RPbox) {}; 
  \end{scope}
 
\end{tikzpicture}
\caption{TwistNet uses a context encoder to produce a table of embeddings for every coordinate and every possible state value, independently of the current state $z_t$. The state $z_t$ then selects one embedding per coordinate, and the selected embeddings are pooled into a single representation of the full state. To evaluate the twist at any single-site modification of $z_t$, TwistNet reuses this pooled representation and updates only the contribution of the modified coordinate, enabling parallel evaluation of all single-site neighboring states.}\label{fig:twistnet}
\end{figure*}

\subsection{Efficient twist parameterization.}\label{sec:efficient}
Parameterizing the twist–induced proposal in \eqref{eq:twist_approx_kernel} is challenging because, at each step and for a given state $z_t$, we need all ratios
$$
s^\psi_{i,t}(v,z_t)
= \frac{h_t^\psi(z_t^{i\to v})}{h_t^\psi(z_t)},
\qquad i \in \gI,\ v \in \mathcal V,
$$
that is, the twist evaluated at every single–site modification $z_t^{i\to v}$ of $z_t$. A naïve implementation would call $h_t^\psi$ separately on each configuration $z_t^{i\to v}$, requiring $O(dV)$ forward passes per time step.
To make this structure explicit, define for each configuration $z\in\mathcal Z$ the matrix
$$
H_t^\psi(z) \in \mathbb R_+^{d\times V},
\qquad
H_t^\psi(z)_{[i,v]} \coloneqq h_t^\psi\bigl(z^{i\to v}\bigr).
$$
In particular, $H_t^\psi(z)_{[i,z^i]} = h_t^\psi(z)$ for all $i\in \gI$. By construction, replacing the $i$-th coordinate of $z$ by any label $u \in \mathcal V$ leaves the \textit{row} $H_t^\psi(z)_{[i,\cdot]}$ unchanged:
\begin{equation}\label{eq:invariance}
H_t^\psi(z)_{[i,v]}
= H_t^\psi\bigl(z^{i\to u}\bigr)_{[i,v]},
\qquad u,v\in\mathcal V,\ i\in \gI.    
\end{equation}

This is the functional identity we would like the learned mapping $z \mapsto H_t^\psi(z)$ to satisfy: each entry $H_t^\psi(z)_{[i,v]}$ should depend on $z$ only through the context provided by the other coordinates and the choice of index–value pair $(i,v)$, not on the current value $z^i$ itself.

A direct implementation of $H_t^\psi(z)$ by explicitly forming all $z^{i\to v}$ would be prohibitively expensive. Instead, we directly exploit the product structure of the state space to build an efficient alternative. We first use a \textbf{context encoder} that can depend on any contextual information $c_t$ we wish to make available for the model at time $t$, except for the state $z_t$ itself:
\begin{equation}\label{eq:encoder}
    \Phi_t(c_t) \in \mathbb R^{d\times V\times m}, 
\end{equation}
where $c_t$ can include, for instance, future observations and observation times $(Y_{\ge t},\tau_{\ge t})$, covariates, and positional information if available. We highlight that the encoder in \eqref{eq:encoder} is independent of the current state $z_t$, and produces a tensor of embeddings associated with setting each coordinate $i$ to a state $v$. Based on these embeddings, we can compose the output $h^\psi_t(z)$ by aggregating the representations at the positions indexed by our target state, via a map
\begin{equation}
    \label{eq:rho}
   \log h_t^\psi(z)
= \phi\!\left(\sum_{i\in \gI} \Phi_t(c_t)_{[i,z^i]}\right),
\end{equation}
where $\phi:\mathbb R^m \to \mathbb R$ can be a learned map (e.g., a small MLP), a construction mirroring the DeepSets model \citep{zaheer2017deep}. Let $S_t(z) \coloneqq \sum_{j\in \gI} \Phi_t(c_t)_{[j,z^j]}$ denote the pooled representation of the current configuration. For each single–site update $z^{i\to v}$, we can reuse $S_t(z)$ and adjust only the contribution of coordinate $i$:
\begin{equation}
   \log h_t^\psi\bigl(z^{i\to v}\bigr)\!
=\! \phi\Bigl(S_t(z) + \Phi_t(c_t)_{[i,v]} - \Phi_t(c_t)_{[i,z^i]}\Bigr),
\end{equation}
an operation trivial to parallelize. Moreover, the identity in \eqref{eq:invariance} holds exactly for this parameterization: replacing $z^i$ by any $u$ changes both $S_t(z)$ and the subtraction $-\Phi_t(c_t)_{[i,z^i]}$ in a way that cancels out, leaving the argument of $\phi$ (and thus $H_t^\psi(\cdot)_{[i,v]}$) invariant for every $v$.

We refer to this architecture, which (i) amortizes the cost of encoding contextual information across all $(i,v)$ pairs and (ii) enforces the desired invariance of \eqref{eq:invariance}, as \textbf{TwistNet}. Figure \ref{fig:twistnet} illustrates this construction: the context encoder computes $\Phi_t(c_t)$ once, the current state $z_t$ selects one embedding per coordinate to form $S_t(z_t)$, and each single-site modification $z_t^{i \to v}$ is evaluated by replacing only the $i$-th contribution in the pooled representation.

\paragraph{Twist learning.}\label{sec:twist_learning}
Multiple objective functions have been proposed to learn the twist, from approaches based on consistency inspired by reinforcement learning \citep{heng2020controlled, lawson2018twisted}, to density ratio estimation \citep{lawson2022sixo, lawson2023x}. Recently, \cite{zhao2024probabilistic} proposed to minimize the mass-covering forward KL divergence with respect to the true posterior in the context of autoregressive language models. Forward KL objectives had previously been proposed for learning the proposal distribution in SMC by \cite{gu2015neural, lawson2023x}. We adapt this approach to our setting, and learn an amortized twist function by optimizing $D_\text{KL}(P^\star_\theta \,||\, P^\psi_\theta)$, which is proportional to
\begin{align}\label{eq:fwd_kl}
    &\mathbb{E}_{P^\star_\theta}\!\bigg[\sum_{i\in \gI}\!\bigg(\!\int_{0}^T\!\!r^{\theta,\psi}_{i,t}(Z^i_t\midd Z_t)dt +\!\!\! \!\!\!\!\sum_{u:Z^i_u\neq Z^i_{u^-}}\!\!\!\!\!\!\log s^{\theta,\psi}_{i,u}(Z_u^i,Z_{u^-})\! \bigg)\!\bigg]\!\nonumber\\
    &+\mathbb{E}_{p_0^\star}\left[\log q_0^\psi(Z_0)\right],
\end{align}
where $s^{\theta,\psi}_{i,t}(v,z) = h^\psi_t(z^{i\to v})/h^\psi_t(z)$ 
is the concrete score, and $h^\psi_t$ may depend on future observations 
$y_{>t}$.%
\paragraph{Wake-sleep.} In order to compute a tractable approximation to \eqref{eq:fwd_kl}, we employ the strategy proposed in \cite{zhao2024probabilistic} and in \textit{wake-sleep} algorithms \citep{hinton1995wake, bornschein2014reweighted, le2020revisiting, mcnamara2024sequential}, and perform ancestral sampling of the trajectory from the prior, and of the observations from the emission distribution. We obtain the time-discretized \textit{sleep objective}
\begin{align}\label{eq:approx_fwd_kl}
    \mathcal{L}_\mathrm{s}&(\psi;\,z_{t_0:t_M}, y_{1:K},\theta) = \\
    &-\log q^\psi_0(z_0\mid y_{1:K}) 
    - \sum_{m=0}^{M-1}\sum_{i\in \gI}\Bigl[
    \Delta_{t_{m+1}}\, r^{\theta,\psi}_{i,t_m}(z^i_{t_m}\mid z_{t_m}) \nonumber\\
    &+\;\boldsymbol{1}[z^i_{t_m}\neq z^i_{t_{m+1}}]
    \log s^{\psi}_{i,t_m}(z^i_{t_{m+1}},z_{t_m})\Bigr],
\end{align}
where $0=t_0<t_1<\dots<t_M=T$ is a time grid containing $\{\tau_k\}_{k=1}^K$ with $\Delta_{t_m} = t_m-t_{m-1}$, and $q^\psi_0$ is an approximate posterior initial distribution. Note that, although omitted for notational simplicity, $s^\psi_{i,t_m}$ is conditioned on future observations $y_{>t_m}$.
We optimize this objective with respect to the parameters $\psi$, while holding $\theta$ fixed.
 Note that optimizing \eqref{eq:approx_fwd_kl} does not require backpropagating through the sampled trajectories, avoiding the need to use reparameterization \citep{jang2017categorical} or REINFORCE \citep{williams1992simple}.

We alternate $\psi-$updates on the objective in \eqref{eq:approx_fwd_kl} with $\theta-$updates approximating \eqref{eq:mle} based on trajectories sampled using tSMC, similarly to \cite{lawson2023x} and \cite{mcnamara2024sequential}.
 For the $\theta$-updates, we use the discretized log path 
density of the prior
\begin{align}\label{eq:disc_path}
    \log \hat{P}_\theta&(z_{[0,T]}) = \nonumber\\
    &\log p_0^\theta(z_0) 
    + \sum_{m=0}^{M-1}\sum_{i\in \gI}\Bigl[
        \Delta_{t_{m+1}} r^{\theta}_{i,t_m}(z^i_{t_m}\mid z_{t_m}) \nonumber\\
        &+\;\boldsymbol{1}[z^i_{t_m}\neq z^i_{t_{m+1}}]
        \log r^{\theta}_{i,t_m}(z^i_{t_{m+1}}\mid z_{t_m})\Bigr],
\end{align}
to form the wake gradient estimator
\begin{align}\label{eq:approx_mle}
    &\widehat{\nabla_\theta \mathcal{L}_\mathrm{w}}\big(\theta;\,\big\{z^{(s)}_{t_0:t_M},\bar w_T^{(s)}\big\}_{s=1}^S, y_{1:K}\big) = \\
    &\sum_{s=1}^S \bar{w}^{(s)}_T \left[
    \sum_{k=1}^K \nabla_\theta\log G_{\tau_k,\theta}(z^{(s)}_{\tau_k}) 
    + \nabla_\theta\log \hat P_\theta(z^{(s)}_{[0,T]}) \right],
\end{align}
where $\bar w^{(s)}_T$ are SMC weights as in \eqref{eq:smc_weights}.
We note that the estimator in \eqref{eq:approx_mle} is asymptotically consistent for the gradient of the maximum likelihood objective \citep{lawson2023x}.
Simplified pseudocode is presented in Algorithm \ref{algo:ws}, and the detailed training algorithm and its complexity are reported in Appendix~\ref{appdx:algo}.

\begin{algorithm}[t]
\small
\caption{Wake--sleep with tSMC for latent IPS}
\label{algo:ws}
\begin{algorithmic}[1]
\STATE \textbf{Inputs:} Dataset $\mathcal{D}$; time grid $0=t_0<\dots<t_M=T$; 
       particles $S$; batch size $B$
\STATE Initialize $\theta$, $\psi$
\REPEAT
  \STATE \textbf{Sleep phase:} \textcolor{gray}{\textit{\# update $\psi$, hold $\theta$ fixed}}
  \FOR{$b = 1, \dots, B$}
    \STATE Simulate $z^{(b)}_{t_0:t_M} \sim P_\theta$ via Euler steps
    \STATE Simulate $\tilde{y}^{(b)}_k \sim p_\theta(\cdot \mid z^{(b)}_{\tau_k})$, \; $k=1,\dots,K$
    \STATE Compute $\ell^{(b)}_\mathrm{s}=\gL_\mathrm{s}(\psi;\,z^{(b)}_{t_0:t_M}, \tilde{y}^{(b)}_{1:K},\theta)$ via \eqref{eq:approx_fwd_kl} 
  \ENDFOR
  \STATE $\psi \leftarrow \textsc{GradStep}_\psi\!\left(\nabla_\psi \frac{1}{B}\sum_{b=1}^B \ell^{(b)}_\mathrm{s}\right)$
  \STATE \textbf{Wake phase:} \textcolor{gray}{\textit{\# update $\theta$, hold $\psi$ fixed}}
  \FOR{$b = 1, \dots, B$}
    \STATE Sample $(y^{(b)}_{1:K}, \tau^{(b)}_{1:K}) \sim \mathcal{D}$
    \STATE Run $\{z^{(b,s)}_{t_0:t_M}, \bar{w}^{(b,s)}_T\}_{s=1}^S \leftarrow 
           \mathrm{tSMC}(P^\psi_\theta;\, y^{(b)}_{1:K}, \tau^{(b)}_{1:K})$
    \STATE Compute $g_\mathrm{w}^{(b)}\!\! =\!\widehat{\nabla_\theta \gL}^{(b)}_\mathrm{w}\!\!(\theta; \{z^{(b,s)}_{t_0:t_M}, \bar{w}^{(b,s)}_T\}_{s=1}^S, y^{(b)}_{1:K})$ via \eqref{eq:approx_mle}
  \ENDFOR
  \STATE $\theta \leftarrow \textsc{GradStep}_\theta\!\left(\frac{1}{B}\sum_{b=1}^B g_\mathrm{w}^{(b)}\right)$
\UNTIL{convergence}
\end{algorithmic}
\end{algorithm}
\begin{figure}[!t]

  \centering
  
    \includegraphics[width=\linewidth]{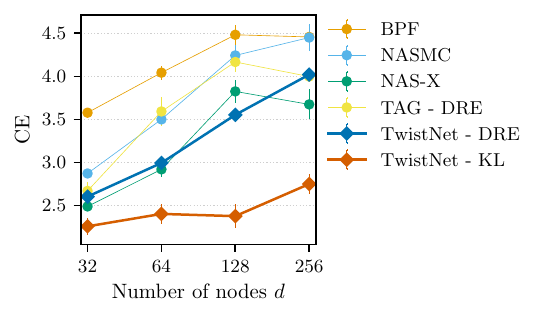}

  \caption{
  Latent trajectory reconstruction, measured by cross-entropy loss on the test set of ground truth trajectories with respect to the posterior approximations. Error bars correspond to two standard errors computed across trajectories.}\label{fig:nll_vs_dim}
\end{figure}

\begin{figure*}[ht]
\centering
\footnotesize
\captionof{table}{Parameter estimates and relative parameter error (RPE) for the SIRS model with 32 nodes, mean $\pm$ 2 standard deviations across 10 random seeds.}
\label{tab:final-param-estimates}
\begin{tabular}{lccccc}
 & $\boldsymbol{\alpha_0}$ & $\boldsymbol{\alpha_1}$ & $\boldsymbol{\beta}$ & $\boldsymbol{\gamma}$ & \textbf{RPE} \\
\hline \\
\textbf{Ground truth} & 0.1 & 1.0 & 0.4 & 0.05 & -- \\
TwistNet - KL & \textbf{0.113} $\pm$ 0.03 & 0.922 $\pm$ 0.08 & 0.393 $\pm$ 0.03 & 0.046 $\pm$ 0.01 & \textbf{0.330} $\pm$ 0.35 \\
TwistNet - DRE & 0.115 $\pm$ 0.03 &\textbf{0.957} $\pm$ 0.15 & 0.384 $\pm$ 0.03 & \textbf{0.048} $\pm$ 0.01 & 0.360 $\pm$ 0.34 \\
TAG - DRE & 0.124 $\pm$ 0.05 & 0.673 $\pm$ 0.43 & 0.377 $\pm$ 0.14 & 0.055 $\pm$ 0.02 & 0.982 $\pm$ 0.68 \\
NAS-X & 0.136 $\pm$ 0.04 & 0.880 $\pm$ 0.10 & \textbf{0.405} $\pm$ 0.04 & 0.058 $\pm$ 0.01 & 0.673 $\pm$ 0.39 \\
NASMC & 0.076 $\pm$ 0.09 & 0.936 $\pm$ 0.26 & 0.444 $\pm$ 0.11 & 0.061 $\pm$ 0.04 & 0.962 $\pm$ 0.56 \\
BPF & 0.791 $\pm$ 1.90 & 1.354 $\pm$ 2.05 & 0.618 $\pm$ 0.24 & 0.001 $\pm$ 0.00 & 9.461 $\pm$ 17.88 \\
\end{tabular}

\vspace{2em}

\includegraphics[width=0.85\linewidth]{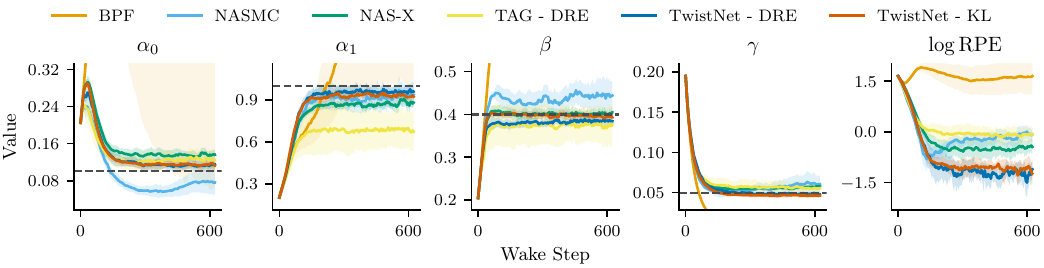}
\vspace{1em}

\captionof{figure}{ Evolution of the parameters and log relative parameter error ($\log\,$RPE) through wake steps updating the parameters $\theta$, for a SIRS model on a graph with 32 nodes.}
\label{fig:params}

\end{figure*}

\section{EXPERIMENTS}\label{sec:sirs}

\subsection{Latent SIRS model}

We study inference in a latent spatial SIRS IPS on graphs with state space $\gZ=\{S,I,R\}^d$ \citep{lanchier2024stochastic}. We simulate a dataset of trajectories using Gillespie's algorithm \citep{gillespie1977exact, wilkinson2018stochastic} with the following local rates:
\begin{gather}
  r_{i,t}(I\mid z) = \bigl(\alpha_0 + \alpha_1\, \textstyle\sum_{j\neq i} a_{ij}\,\sigma(\langle \xi_i,\xi_j\rangle)\,\delta_{z^j,I}\bigr)\,\delta_{z^i,S}, \nonumber \\
  r_{i,t}(R\mid z) = \beta\,\delta_{z^i,I}, \qquad
  r_{i,t}(S\mid z) = \gamma\,\delta_{z^i,R}, \label{eq:sirs_rates}
\end{gather}
where $a_{ij}$ is an adjacency matrix sampled from a graph with fixed expected degree, $\xi_i$ are normalized features acting as edge weights, and $\sigma$ is a logistic function ensuring non-negativity.
We let observation times be sampled uniformly in $[0,T]$ at $K=10$ irregular \emph{snapshots} $\tau_1<\dots<\tau_K$, and we sample observations from an emission distribution masking each node (denoted $y_k=\varnothing$) independently with probability $p_{\mathrm{mask}}=0.5$.

\paragraph{Task. }We conduct two experiments. First, we perform \emph{latent trajectory inference} \citep{eichentropic}: given observations $y_{1:K}$ and parameters $\theta$ at their true values, we measure how accurately our approximate posterior samples match the ground truth trajectory. We focus on scalability with respect to graph dimensionality, quantifying performance via categorical cross-entropy loss averaged over a discretized time grid.
Second, we \emph{estimate parameters} from an arbitrary initialization using only 50 training trajectories. Following Section \ref{sec:twist_learning}, we employ a \textit{wake-sleep} algorithm and evaluate accuracy via individual parameter estimates and total relative parameter error $\sum_j{|\hat \theta_j- \theta_j|}/{|\theta_j|}$ over the four parameters in equation \ref{eq:sirs_rates}.

\paragraph{Methods. }For our TwistNet parameterization, we experiment learning the twist using both our forward KL loss from \eqref{eq:approx_fwd_kl} and the DRE loss introduced by \citep{lawson2022sixo}, using a graph transformer as context encoder and a two-layer MLP as $\rho$ in \eqref{eq:rho}. We consider NASMC \citep{gu2015neural} and NAS-X \citep{lawson2023x}, and parameterize the proposal distribution with a graph transformer as the context encoder used in TwistNet, taking the current state as input and predicting a concrete score field. The twist for NAS-X follows the same parameterization, with output pooled to a scalar. We experiment using TAG \citep{nisonoff_unlocking_2025} with DRE loss, and a regular BPF \citep{doucet2009tutorial}. 
We discuss further details in Appendix \ref{base:details} and \ref{sirs:details}.

\paragraph{Results. }As illustrated in Figure \ref{fig:nll_vs_dim}, we find the combination of TwistNet and the forward KL loss in \eqref{eq:fwd_kl} to scale the best with respect to dimensions for latent trajectory inference. We report results on parameter estimation on a graph with 32 nodes in Table \ref{tab:final-param-estimates}, and display convergence through wake steps in Figure \ref{fig:params}. Additional results are reported in Appendix~\ref{appdx:add_sirs}.

\begin{table*}[t]
\footnotesize
\centering
\caption{Binary cross-entropy loss for reconstruction and prediction of active fire maps, mean $\pm$ 2 standard errors across test trajectories.}
\label{tab:wildfire-results}
\begin{tabular}{lcccc}
\hline
& \multicolumn{2}{c}{\textbf{Full-week covariates}} & \multicolumn{2}{c}{\textbf{Past covariates only}} \\
\cline{2-3} \cline{4-5}
& \textbf{Reconstruction} & \textbf{Prediction} & \textbf{Reconstruction} & \textbf{Prediction} \\
\hline
TwistNet - KL  & \textbf{0.877} $\pm$ 0.36  & \textbf{1.149} $\pm$ 0.41  & \textbf{1.046} $\pm$ 0.389 & \textbf{1.424} $\pm$ 0.499 \\
TwistNet - DRE & 1.985 $\pm$ 0.59           & 1.541 $\pm$ 0.52           & 1.702 $\pm$ 0.532          & 1.649 $\pm$ 0.536 \\
NAS-X          & 2.392 $\pm$ 0.72           & 1.755 $\pm$ 0.53           & 1.787 $\pm$ 0.531          & 1.576 $\pm$ 0.523 \\
NASMC          & 2.008 $\pm$ 0.60           & 1.759 $\pm$ 0.57           & 3.350 $\pm$ 0.897          & 1.992 $\pm$ 0.563 \\
\hline
\end{tabular}
\end{table*}
\begin{figure*}
    \centering
    \includegraphics[width=\linewidth]{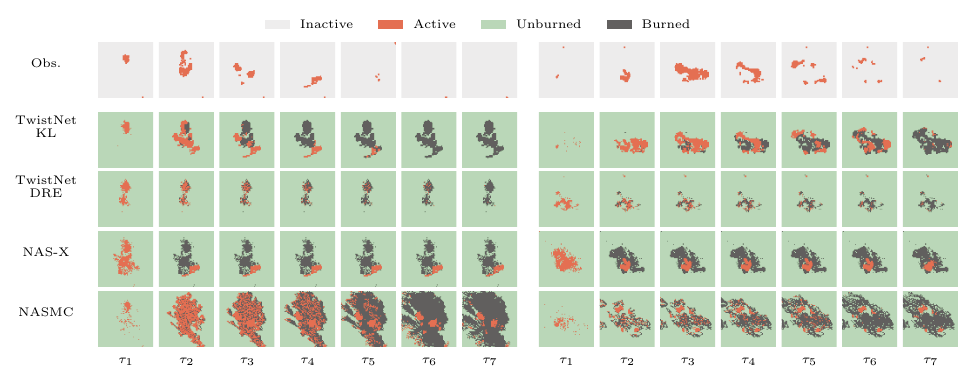}
    \caption{Observed active fire maps from WildFireSpreadTS \citep{gerard2023wildfirespreadts} and snapshots of approximate posterior samples at observation times, for the \emph{full-week covariate} regime. A pixel can only turn active if another one in its Moore neighborhood is, and transitions from unburned to burned are impossible. Observations are noisy, as satellite imagery for wildfires can be obscured by smoke and clouds \citep{schroeder2014new}.}
    \label{fig:wildfire}
\end{figure*}

\subsection{Wildfire trajectories}\label{sec:wildfires}
Forecasting wildfire spread is a problem of critical importance, given the impacts of wildfires and their increasing severity globally in recent years \citep{abatzoglou2016impact,walker2019increasing}. Physics-based models have long existed \citep{rothermel1972mathematical} but are difficult to calibrate given real-world constraints, limiting forecasting utility. With the increasing availability of satellite data, ML approaches for spatio-temporal prediction have emerged (e.g., \cite{coffield2019machine,prapas2021deep,apostolakis2022estimating, gerard2023wildfirespreadts}), yet they can lack physical insight (e.g., local propagation constraints) and have not yet seen widespread application in real-world scenarios. A third class of model in this context is based on the notion of local spreading of fire (across pixels) via stochastic cellular automata (SCAs), closely related to our latent IPS framework.  Prior work in developing SCAs for wildfires has largely focused on relatively simple discrete-time models \citep{grieshop2024data}, often hand-tuned \citep{clarke1994cellular,hargrove2000simulating}.

We model wildfires as latent IPSs, and consider a state space where each pixel in a $64\times64$ grid can be unburned ($U$), active ($A$), or burned ($B$), i.e. $\gZ=\{U,A,B\}^{64\times 64}$. We let the graph be a lattice with Moore neighborhood. We parameterize outgoing, off-diagonal local rates by a single neural network $F_\theta$ mapping to $\R_+^{d\times (V-1)}$, using the UTAE architecture \citep{garnot2021panoptic}. We constrain the dynamics by zeroing out the rate of flipping to $A$ whenever no neighbors are ignited, and do not allow transitions from $U$ to $B$ and vice versa.

\paragraph{Task.}Using the WildFireSpreadTS dataset \citep{gerard2023wildfirespreadts}, we extract 156 week-long trajectories on $64\times64$ grids where at least one pixel is initially active with two subsequent days of activity, split 130-26 into a train and test set. We evaluate reconstruction of active fire maps at observation times (i.e., conditioning on the whole sequence) and prediction (initialized by encoding only the first observation), both measured via binary cross-entropy from the empirical distribution given by 16 particles per datapoint, averaged across time. 
We report results in two regimes: conditioning on VIIRS reflectance channels for the full week, and conditioning only on channels observed up to each timepoint.
\paragraph{Methods.} 
 We use the UTAE architecture for the context encoder in TwistNet, and for proposal and twist in NAS-X and NASMC. We train initial distributions -- both posterior and prior -- using a smaller model with the same architecture. We let our emission distribution be factorized, and let $p(y^i_k = 1 \midd Z_{\tau_k}^i=A) = \sigma(\theta_\text{detect})$, 
where $\theta_\text{detect}$ is a scalar logit and $\sigma$ the logistic function. Due to their poor performance, we do not compare with TAG and BPF. For each method we also train an encoder for the first observation, used in the prediction task, by minimizing a forward KL loss in the \textit{sleep} phase. We report experimental details in Appendix~\ref{wf:details}.
\paragraph{Results.} Table \ref{tab:wildfire-results} reports results both with full-week covariates and with covariates restricted to those available up to each timepoint. TwistNet with the KL loss achieves the lowest average error in both regimes. By examining trajectories from the second best model (TwistNet with a DRE loss) in the full-week regime in Figure \ref{fig:wildfire}, we see a large gap in terms of how natural the trajectories look, and especially in how well the approximate posterior can bridge between observed states. Additional results are reported in Appendix~\ref{appdx:add_wf}.
\section{CONCLUSIONS}
We introduced an efficient posterior inference framework for systems of latent interacting CTMCs, and empirically demonstrated its effectiveness on challenging tasks on latent state inference and parameter learning for both simulated and real data. Our proposed approach outperforms the baselines, and in our experiments we saw this performance gap increase with the dimensionality of the problem under consideration.
A major bottleneck in the training of our method is the cost of simulating trajectories with twisted SMC. Recent advances in latent SDEs have shown that \textit{simulation-free} algorithms are possible \citep{bartosh2025sde, kiyohara2025neural}, and we believe adapting these methods to latent IPSs is an exciting avenue for future work.

\subsubsection*{Acknowledgements}
We thank the reviewers for their feedback on improving the paper. This work was supported by the National Science Foundation under awards NSF 2505006 and NSF 2425932, by the National Institutes of Health under awards R01-LM013344 and R01CA297869, by the Hasso Plattner Institute (HPI) Research Center in Machine Learning and Data Science at UCI, and by funding support from Google and from SAP.

\bibliographystyle{plainnat}
\bibliography{sample,fire_references}

\newpage
\section*{Checklist}

\begin{enumerate}

  \item For all models and algorithms presented, check if you include:
  \begin{enumerate}
    \item A clear description of the mathematical setting, assumptions, algorithm, and/or model. [Yes]
    \item An analysis of the properties and complexity (time, space, sample size) of any algorithm. [Yes]
    \item (Optional) Anonymized source code, with specification of all dependencies, including external libraries. [No]
  \end{enumerate}

  \item For any theoretical claim, check if you include:
  \begin{enumerate}
    \item Statements of the full set of assumptions of all theoretical results. [Yes]
    \item Complete proofs of all theoretical results. [Yes]
    \item Clear explanations of any assumptions. [Yes]     
  \end{enumerate}

  \item For all figures and tables that present empirical results, check if you include:
  \begin{enumerate}
    \item The code, data, and instructions needed to reproduce the main experimental results (either in the supplemental material or as a URL). [No]
    \item All the training details (e.g., data splits, hyperparameters, how they were chosen). [Yes]
    \item A clear definition of the specific measure or statistics and error bars (e.g., with respect to the random seed after running experiments multiple times). [Yes]
    \item A description of the computing infrastructure used. (e.g., type of GPUs, internal cluster, or cloud provider). [Yes]
  \end{enumerate}

  \item If you are using existing assets (e.g., code, data, models) or curating/releasing new assets, check if you include:
  \begin{enumerate}
    \item Citations of the creator If your work uses existing assets. [Yes]
    \item The license information of the assets, if applicable. [Not Applicable]
    \item New assets either in the supplemental material or as a URL, if applicable. [Not Applicable]
    \item Information about consent from data providers/curators. [Not Applicable]
    \item Discussion of sensible content if applicable, e.g., personally identifiable information or offensive content. [Not Applicable]
  \end{enumerate}

  \item If you used crowdsourcing or conducted research with human subjects, check if you include:
  \begin{enumerate}
    \item The full text of instructions given to participants and screenshots. [Not Applicable]
    \item Descriptions of potential participant risks, with links to Institutional Review Board (IRB) approvals if applicable. [Not Applicable]
    \item The estimated hourly wage paid to participants and the total amount spent on participant compensation. [Not Applicable]
  \end{enumerate}

\end{enumerate}

\clearpage
\appendix
\thispagestyle{empty}

\onecolumn

\runningtitle{Efficient Inference for Coupled Hidden Markov Models in Continuous Time and Discrete Space}

\onecolumn
\aistatstitle{Efficient Inference for Coupled Hidden Markov Models \\ in Continuous Time and Discrete Space: 
Supplementary Materials}

\section{Background}\label{appdx:background}
\paragraph{Inference for CTMCs.}

Inference methods for CTMCs have been extensively studied. Maximum likelihood estimation for time-homogeneous CTMCs is discussed in \citet{jackson2011multi,bladt2005statistical,mcgibbon2015efficient}. Expectation-maximization techniques for continuous-time hidden Markov models can be found in \citet{bureau2003applications,jackson2011multi,liu2015efficient}. Bayesian approaches include Markov chain Monte Carlo methods \citep{boys2008bayesian, hobolth2009simulation, rao2013fast} and variational methods. The latter include mean-field \citep{opper2007variational, cohn2010mean}, moment-based methods \citep{wildner2019moment}, combinations with MCMC \citep{zhang2017collapsed}, and extensions to hybrid processes \citep{kohs2021variational}. More recent methods include black-box variational inference with neural networks \citep{seifner2023neural}, foundation models \citep{berghaus2024neurips}, and expectation propagation \citep{eichentropic}. %

\paragraph{Related work.} Another directly related line of research focuses on simulation methods for Markov bridges, notably \cite{hobolth2009simulation, huang2016reconstructing, golightly2019efficient, corstanje2023conditioning, corstanje2025guided}. While less directly related, it is worth noting recent work discrete flow matching and diffusion methods based on CTMCs \citep{sun2023score,meng2022concrete,campbell2022continuous, igashov2023retrobridge, lou2023discrete, campbellgenerative}, as well as discrete neural samplers \citep{holderriethleaps, zhu2025mdns}. In particular, reward-guided generation for discrete diffusion models targets the posterior path measure of a CTMC, given by trajectories from a prior generative model tilted by a reward function at the endpoint $T$ \citep{wangfine2025, rectorsteering}, see \citep{uehara2025inference} for a review. While these methods have shown remarkable performance on posterior inference tasks, they are specific to the processes considered by discrete diffusion models.

\section{Proofs}\label{appdx:proofs}
\subsection{Proof of Proposition \ref{prop:tv-compact}}\label{proof:tv-compact}

By the functional form of the rates, reaching any $\tilde z$ with at least two coordinates changed requires at least two jumps.
By (A1) the jump times are dominated by an homogeneous Poisson process of rate $\bar\lambda$, hence on $[t,t+\Delta_t]$
\[
\mathbb{P}\big(\text{$\ge2$ jumps}\big)\;\le\;\mathbb{P}\big(U\ge2\big)
\;=\; O(\Delta_t^2),\quad U\sim\text{Poisson}(\bar\lambda\Delta_t).
\]
Therefore,
\[
\sum_{\substack{\tilde z:\,|\{i:\tilde z^i\neq z^i\}|\ge 2}} p_{t,\Delta_t}(\tilde z\mid z)=O(\Delta_t^2).
\]
As a next step, considering a starting state $z$, we partition the state space $\gZ$ into three cases and denote them as
$S_0=\{z\}$ (no change), $S_1=\{z^{i\to v}: i\in\gI,\ v\neq z^i\}$ (single flips), and 
$S_{\ge 2}=\{\tilde z:|\{i:\tilde z^i\neq z^i\}|\ge 2\}$ (multi flips).

By (A3), for any $(t,z,i,v)$,
\[
\left|\Delta_t\,r^i_t(v\mid z) - \int_t^{t+\Delta_t}\! r^i_{u}(v\mid z)\,du \right| \leq \int_t^{t+\Delta_t} \big|r^i_t(v\mid z) - r^i_{u}(v\mid z)\big| du \le \int_t^{t+\Delta_t}L(u-t)du = \frac{L}{2}\,\Delta_t^2,
\]
and similarly
\[
\left|\Delta_t\, \lambda_t(z) - \int_t^{t+\Delta_t}\! \lambda_{u}(z)\,du \right|
\le \frac{L_\lambda}{2}\,\Delta_t^2,
\qquad L_\lambda:=  d\times (V-1)\times L.
\]
Hence the following first–order Taylor expansions holds with uniform $O(\Delta_t^2)$ remainders:
\begin{align*}
k_{t,\Delta_t}(z'\mid z) =
\begin{cases}
1 - \int^{t+\Delta_t}_t \lambda_u(z)du + O(\Delta_t^2) = 1 - \lambda_t(z)\,\Delta_t \;+\; O(\Delta_t^2) & z'\in S_0 \\
    \int^{t+\Delta_t}_t r^i_u(v\mid z)du \;+\; O(\Delta_t^2)= r^i_t(v\mid z)\,\Delta_t \;+\; O(\Delta_t^2) &  z'\in S_1\\
    O(\Delta_t^2) & {z'\in S_{\ge2}}
\end{cases}
\end{align*}
Expanding \eqref{eq:prod-kernel}:
\begin{align*}
q_{t,\Delta_t}(z'\mid z)
= \begin{cases}
1 - \lambda_t(z)\,\Delta_t \;+\; O(\Delta_t^2) & z'\in S_0 \\
    r^i_t(v\mid z)\,\Delta_t \;+\; O(\Delta_t^2) &  z'\in S_1\\
    O(\Delta_t^2) & z'\in S_{\ge 2}
\end{cases}
\end{align*}
Recall $\|\mu-\nu\|_{\mathrm{TV}}=\tfrac12\sum_{z'}|\mu(z')-\nu(z')|$. Then,
\begin{align*}
\sum_{z'\in S_0} \big| k_{t,\Delta_t}(z'\mid z)-q_{t,\Delta_t}(z'\mid z) \big|
&= \big| k_{t,\Delta_t}(z\mid z)-q_{t,\Delta_t}(z\mid z) \big|
\;\le\; C_0\,\Delta_t^2,\\[2pt]
\sum_{z'\in S_1} \big| k_{t,\Delta_t}(z'\mid z)-q_{t,\Delta_t}(z'\mid z) \big|
&\le\; \sum_{i}\sum_{v\neq z^i} C_1\,\Delta_t^2
\;=\; d\,(|\gV|-1)\,C_1\,\Delta_t^2,\\[2pt]
\sum_{z'\in S_{\ge 2}} \big| k_{t,\Delta_t}(z'\mid z)-q_{t,\Delta_t}(z'\mid z) \big|
&\le\; \sum_{z'\in S_{\ge 2}} k_{t,\Delta_t}(z'\mid z) 
\;+\; \sum_{z'\in S_{\ge 2}} q_{t,\Delta_t}(z'\mid z)
\;\le\; C_2\,\Delta_t^2,
\end{align*}
where $C_0,C_1$ arise from the $O(\Delta_t^2)$ remainders in the first-order expansions, and $C_2$ from the multi-flip probabilities. Therefore,
\[
\big\|Q_{t,\Delta_t}(\cdot\mid z)-K_{t,\Delta_t}(\cdot\mid z)\big\|_{\mathrm{TV}}
\;=\;\frac12\sum_{z'} \big| k_{t,\Delta_t}(z'\mid z)-q_{t,\Delta_t}(z'\mid z) \big|
\;\le\; \frac{C_0 + d(|\gV|-1)C_1 + C_2}{2}\;\Delta_t^2
\;\le\; C\,\Delta_t^2,
\]
with $C<\infty$.

\subsection{Proof of Proposition \ref{prop:doob_h_transform}}\label{proof:doob_h_transform}

    We proceed by analyzing the transition kernel of $P^\star$ in an arbitrary interval $[s,t]\subseteq[0,T]$.

    Let
    \begin{equation}
        M_T \coloneq \frac{P^\star(dz_{[0,T]})}{P(dz_{[0,T]})} = \frac{\prod^K_{k=1}G_{\tau_k}(z_{\tau_k})}{\E_P[\prod^K_{k=1}G_{\tau_k}(z_{\tau_k})]},
    \end{equation}
    then, let the \textit{filtered} path measure $P^\star_t$ be a restriction of $P^\star$ to the filtration $\gH_t$, where $0<t<T$, 
    and denote
    \begin{equation}
        \frac{P^\star_t(dz_{[0,t]})}{P_t(dz_{[0,t]})} = M_t.
    \end{equation}
    For an event $B\in \gH_t$, by a simple application of the Radon-Nykodym theorem and the tower property we can write
    \begin{equation}
        P^\star_t(B) = P^\star(B) = \E_{P}[1_B M_T] = \E_P[1_B \E_P[M_T\midd \gH_t]] = \E_{P_t}[1_B \E_{P_t}[M_T\midd \gH_t]],
    \end{equation}
    where the last step follows from $\E_P[M_T\midd \gH_t]$ being measurable with respect to $\gH_t$. Hence,
    \begin{equation}
        M_t =\E_{P_t}[M_T\midd \gH_t]= \frac{1}{\E_P[\prod^K_{k=1}G_{\tau_k}(z_{\tau_k})]}\E_P\left[\prod^K_{k=1}G_{\tau_k}(Z_{\tau_k}) \bigg| \gH_t\right].
    \end{equation}
    By change of measure under conditional expectation, it follows that
    \begin{equation}
        \E_{P^\star}[f(Z_t)\mid Z_s = z] = \frac{\E_{P}[f(Z_t) M_t\mid Z_s = z]}{\E_{P}[M_t\mid Z_s = z]} = \frac{\E_{P}[f(Z_t) h^\star_t(Z_t)\mid Z_s = z]}{h^\star_s(z)},
    \end{equation}
    where $h^\star$ is the look-ahead function in \eqref{eq:lookahead}.
    By definition, we can express the generator $\gL^\star_t$ of $P^\star_t$ as
    \begin{align}
        \gL^\star_t(f)(z) 
        &=\lim_{\Delta_t \rightarrow 0} \frac{\E_{P^\star}[f(Z_{t+\Delta_t})\mid Z_t = z] - f(z)}{\Delta_t} \\
        &= \lim_{\Delta_t \rightarrow 0} \frac{\E_{P}\left[f(Z_{t+\Delta_t}) \frac{h^\star_{t+\Delta_t}(Z_{t+\Delta_t})}{h^\star_t(z)}\mid Z_t = z\right] - f(z)}{\Delta_t} \label{eq:generator_star}
    \end{align}
    Moreover, we can approximate $h^\star_{t+\Delta_t}(z)$ for $t\in[\tau_k,\tau_{k+1}-\Delta_t),\,k\in[1:K]$ using a Taylor expansion around time $t$
    \begin{align}
        h^\star_{t+\Delta_t}(z) &= h^\star_t(z) +\Delta_t \frac{\partial h^\star_t(z)}{\partial t} + o(\Delta_t) \\
        &= h^\star_t(z) -\Delta_t \sum_{i,v\neq z^i} r^i_t(v\midd z) \left[ h^\star_t(z^{i\rightarrow v}) - h^\star_t(z) \right] + o(\Delta_t),  \label{eq:taylor}
    \end{align}
    where the last line follows from Kolmogorov backward equation \citep{norris}:
    \begin{equation*}
        \frac{\partial h^\star_t(z)}{\partial t} = -\mathcal{L}_t (h^\star_t)(z) = -\sum_{i,v\neq z^i} r^i_t(v\midd z) \left[ h^\star_t(z^{i\rightarrow v}) - h^\star_t(z) \right].
    \end{equation*}
    For small $\Delta_t$, we can express $\E_{P}\left[f(Z_{t+\Delta_t}) \frac{h^\star_{t}(Z_{t+\Delta_t})}{h^\star_t(z)}\mid Z_t = z\right]$ using the law of total expectation, where we split the expectation based on the number of jumps in the interval $[t,t+\Delta_t]$. 
    \begin{align}
        &\E_{P}\left[f(Z_{t+\Delta_t}) h^\star_{t}(Z_{t+\Delta_t})\mid Z_t = z\right] \\
        &= f(z)h^\star_t(z)\bigg(1 - \Delta_t \sum_{i,v\neq z^i} r^i_t(v\midd z)\bigg) +  \sum_{i,v\neq z^i} f(z^{i\rightarrow v}) h^\star_t(z^{i\rightarrow v})\Delta_t r^i_t(v\midd z)  + o(\Delta_t) \\
        &= f(z)h^\star_t(z) + \Delta_t \sum_{i,v\neq z^i} r^i_t(v\midd z)  \left[ f(z^{i\rightarrow v}) h^\star_t(z^{i\rightarrow v}) - f(z)h^\star_t(z)\right] + o(\Delta_t).\label{eq:total_exp}
    \end{align}
    Combining \eqref{eq:taylor} and \eqref{eq:total_exp}, we obtain 
    \begin{align}
        &\E_{P}\left[f(Z_{t+\Delta_t}) h^\star_{t+\Delta_t}(Z_{t+\Delta_t})\mid Z_t = z\right] \\
        &= f(z)h^\star_t(z) + \Delta_t \sum_{i,v\neq z^i} r^i_t(v\midd z)h^\star_t(z^{i\rightarrow v})  \left[ f(z^{i\rightarrow v})  - f(z)\right] + o(\Delta_t).  \label{eq:approx_future_fh}
    \end{align}
    Then, plugging \eqref{eq:approx_future_fh} back into \eqref{eq:generator_star}, we get 
    \begin{align}
        \gL^\star_t(f)(z) &= \lim_{\Delta_t \rightarrow 0} \frac{\Delta_t \sum_{i,v\neq z^i} r^i_t(v\midd z)\frac{h^\star_t(z^{i\rightarrow v})}{h^\star_t(z)}  \left[ f(z^{i\rightarrow v})  - f(z)\right] + o(\Delta_t)  }{\Delta_t} \\
        &= \sum_{i,v\neq z^i} r^i_t(v\midd z)\frac{h^\star_t(z^{i\rightarrow v})}{h^\star_t(z)}  \left[ f(z^{i\rightarrow v})  - f(z)\right] \label{eq:fk_ips}
    \end{align}
    By inspection, we recognize in \eqref{eq:fk_ips} the generator of an IPS, with local rates 
    \begin{equation}
        r^{\star,i}_t(v\midd z) \coloneq 
        \begin{cases}
            r^i_t(v\midd z)\frac{h^\star_t(z^{i\rightarrow v})}{h^\star_t(z)}, &v\neq z^i, \\
            -\sum_{u\neq v}r^i_t(u\midd z)\frac{h^\star_t(z^{i\rightarrow u})}{h^\star_t(z)} , & v=z^i.
        \end{cases}
    \end{equation}
    The initial distribution follows from a simple application of the Bayes theorem, and is equal to 
    \begin{equation}
        p_0^\star (z) = \frac{p_0(z) h^\star_0(z)}{\E_{p_0}[h_0(Z_0)]}.
    \end{equation}

\subsection{Fisher's identity} \label{proof:fisher}
Let $\mu$ be a $\theta$-independent dominating measure on the path space $D([0,T],\gZ)$. With a slight abuse of notation, we denote the density of paths under $P_\theta(dZ_{[0,T]})$ as 
$P_\theta(Z_{[0,T]})={P_\theta(dZ_{[0,T]})}/{\mu}(dZ_{[0,T]})$. Moreover, we assume conditions hold to interchange differentiation and integration, see \citet[Thm.~6.28]{klenke2008probability} for details. 

We follow the same steps as \cite{lawson2023x}, adapting them to our setup. This amounts to using the derivative of the logarithm and the log-derivative trick inside the expectation:
\begin{align*}
&\nabla_\theta \mathcal{L}(\theta) \\
     &= \nabla_\theta\log \mathbb{E}_{P_\theta}\Bigg[ \prod_{k=1}^K G_{\tau_k,\theta}(Z_{\tau_k}) \Bigg] \\
    &= \frac{1}{\mathbb{E}_{P_\theta}\big[ \prod_{k=1}^K G_{\tau_k,\theta}(Z_{\tau_k}) \big]} \nabla_\theta\int_{D([0,T],\gZ)}  \left(\prod_{k=1}^K G_{\tau_k,\theta}(Z_{\tau_k})\right) P_\theta(Z_{[0,T]})\mu(dZ_{[0,T]}) \\
    &= \frac{1}{\mathbb{E}_{P_\theta}\big[ \prod_{k=1}^K G_{\tau_k,\theta}(Z_{\tau_k}) \big]} \int_{D([0,T],\gZ)}  \nabla_\theta\left(\left(\prod_{k=1}^K G_{\tau_k,\theta}(Z_{\tau_k})\right) P_\theta(Z_{[0,T]})\right)\mu(dZ_{[0,T]}) \\
    &= \frac{1}{\mathbb{E}_{P_\theta}\big[ \prod_{k=1}^K G_{\tau_k,\theta}(Z_{\tau_k}) \big]} \int_{D([0,T],\gZ)}  \left( \sum_{k=1}^K\nabla_\theta \log G_{\tau_k,\theta}(Z_{\tau_k}) +\nabla_\theta \log P_\theta(Z_{[0,T]})\right)\!\!\left(\prod_{k=1}^K G_{\tau_k,\theta}(Z_{\tau_k})\right) P_\theta(dZ_{[0,T]}) \\
    &= \mathbb{E}_{P^\star_\theta}\left[\sum_{k=1}^K\nabla_\theta \log G_{\tau_k,\theta}(Z_{\tau_k}) +\nabla_\theta \log P_\theta(Z_{[0,T]})\right],
\end{align*}
where 
\begin{equation*}\label{eq:theta_star}
    P^\star_\theta(dZ_{[0,T]}) \;\propto\; 
    \Bigg(\prod_{k=1}^K G_{\tau_k}(Z_{\tau_k})\Bigg) P_\theta(dZ_{[0,T]}).
\end{equation*}

\subsection{Proof of Theorem \ref{thm:ess-final}} \label{proof:ess-final}

\begin{lemma}\label{lem:chisq_bound}
Consider a twisted IPS with local rates $r^{\theta,\psi}_{i,t}(v\mid z) = r^\theta_{i,t}(v\mid z) s^\psi_{i,t}(v,z)$ where $s^\psi_{i,t}(v,z) = \frac{h^\psi_t(z^{i\to v})}{h^\psi_t(z)}$. Let $\bar{p}^{\theta,\psi}_{t,\Delta_t}$ be its true transition kernel and $q^{\theta,\psi}_{t,\Delta_t}$ be the Euler approximation:
\begin{equation*}
q^{\theta,\psi}_{t,\Delta_t}(z\mid z_t) = \prod_{i\in\mathcal{I}}\left(\delta_{z^i, z^i_t} + \Delta_t r^{\theta,\psi}_{i,t}(z^i\mid z_t)\right).
\end{equation*}
Under assumptions (A1)-(A3) for the twisted rates (which hold when $h^\psi$ is Lipschitz continuous in time), there exists $C<\infty$ such that:
\begin{equation*}
\chi^2(\bar{p}^{\theta,\psi}_{t,\Delta_t}(\cdot\mid z_t) \| q^{\theta,\psi}_{t,\Delta_t}(\cdot\mid z_t)) \leq C\Delta_t^2,
\end{equation*}
uniformly in $t\in[0,T]$ and $z_t\in\mathcal{Z}$.
\end{lemma}

\begin{proof}
Recall that $\chi^2(p \| q) = \sum_z \frac{(p(z) - q(z))^2}{q(z)}$. We partition $\mathcal{Z}$ as in Proposition \ref{prop:tv-compact}:
\begin{itemize}
    \item $S_0 = \{z_t\}$ (no change)
    \item $S_1 = \{z_t^{i\to v}: i\in\mathcal{I}, v\neq z_t^i\}$ (single flips)
    \item $S_{\geq 2} = \{z': |\{i: z'^i \neq z_t^i\}| \geq 2\}$ (multi flips)
\end{itemize}

Let $\lambda^{\theta,\psi}_t(z_t) = \sum_{i,v\neq z_t^i} r^{\theta,\psi}_{i,t}(v\mid z_t)$. From the proof of Proposition \ref{prop:tv-compact}, under (A1)-(A3) we have the first-order expansions:
\begin{equation}\label{eq:expansions}
\bar{p}^{\theta,\psi}_{t,\Delta_t}(z'\mid z_t) = \begin{cases}
1 - \lambda^{\theta,\psi}_t(z_t)\Delta_t + O(\Delta_t^2) & z' \in S_0 \\
r^{\theta,\psi}_{i,t}(v\mid z_t)\Delta_t + O(\Delta_t^2) & z' = z_t^{i\to v} \in S_1 \\
O(\Delta_t^2) & z' \in S_{\geq 2}
\end{cases}
\end{equation}
and
\begin{equation}\label{eq:euler_expansions}
q^{\theta,\psi}_{t,\Delta_t}(z'\mid z_t) = \begin{cases}
1 - \lambda^{\theta,\psi}_t(z_t)\Delta_t + O(\Delta_t^2) & z' \in S_0 \\
r^{\theta,\psi}_{i,t}(v\mid z_t)\Delta_t + O(\Delta_t^2) & z' = z_t^{i\to v} \in S_1 \\
O(\Delta_t^2) & z' \in S_{\geq 2}
\end{cases}
\end{equation}

\paragraph{Case 1: $z' \in S_0$.}
\begin{align*}
\frac{(\bar{p}^{\theta,\psi}(z_t) - q^{\theta,\psi}(z_t))^2}{q^{\theta,\psi}(z_t)} &= \frac{(O(\Delta_t^2))^2}{1 - \lambda^{\theta,\psi}_t(z_t)\Delta_t + O(\Delta_t^2)} = \frac{O(\Delta_t^4)}{1 + O(\Delta_t)} = O(\Delta_t^4).
\end{align*}

\paragraph{Case 2: $z' \in S_1$.}
For $z' = z_t^{i\to v}$:
\begin{align*}
\frac{(\bar{p}^{\theta,\psi}(z') - q^{\theta,\psi}(z'))^2}{q^{\theta,\psi}(z')} &= \frac{(O(\Delta_t^2))^2}{r^{\theta,\psi}_{i,t}(v\mid z_t)\Delta_t + O(\Delta_t^2)} = \frac{O(\Delta_t^4)}{\Theta(\Delta_t)} = O(\Delta_t^3).
\end{align*}
Summing over $S_1$ (which has $|S_1| = d(|\mathcal{V}|-1)$ elements):
\begin{equation*}
\sum_{z' \in S_1} \frac{(\bar{p}(z') - q(z'))^2}{q(z')} = d(|\mathcal{V}|-1) \cdot O(\Delta_t^3) = O(\Delta_t^3).
\end{equation*}

\paragraph{Case 3: $z' \in S_{\geq 2}$.}
\begin{equation*}
\frac{(\bar{p}(z') - q(z'))^2}{q(z')} = \frac{(O(\Delta_t^2))^2}{O(\Delta_t^2)} = O(\Delta_t^2).
\end{equation*}
Since $|S_{\geq 2}| \leq |\mathcal{Z}| < \infty$:
\begin{equation*}
\sum_{z' \in S_{\geq 2}} \frac{(\bar{p}(z') - q(z'))^2}{q(z')} = O(\Delta_t^2).
\end{equation*}

\paragraph{Total chi-squared divergence.}
Combining all cases:
\begin{equation*}
\chi^2(\bar{p}^{\theta,\psi} \| q^{\theta,\psi}) = O(\Delta_t^4) + O(\Delta_t^3) + O(\Delta_t^2) = O(\Delta_t^2).
\end{equation*}
Therefore, there exists $C < \infty$ such that $\chi^2(\bar{p}^{\theta,\psi}_{t,\Delta_t} \| q^{\theta,\psi}_{t,\Delta_t}) \leq C\Delta_t^2$ uniformly in $t$ and $z_t$.
\end{proof}

As the first step in proving Theorem \ref{thm:ess-final}, we decompose the importance weight using auxiliary distributions. Define:
\begin{itemize}
    \item $p^{\theta,\psi}_{t,\Delta_t}(z\mid z_t) \propto p^\theta_{t,\Delta_t}(z\mid z_t) h^\psi_{t+\Delta_t}(z) G_{t+\Delta_t}(z)^{\mathbf{1}_{t+\Delta_t\in\{\tau_k\}}}$ (twisted with approximate lookahead)
    \item $\bar{p}^{\theta,\psi}_{t,\Delta_t}(z\mid z_t) \propto p^\theta_{t,\Delta_t}(z\mid z_t) h^\psi_{(t+\Delta_t)^-}(z)$ (twisted without potential)
\end{itemize}

The importance weight decomposes as:
\begin{equation}\label{eq:iw_decomp}
\frac{p^{\theta,\star}_{t,\Delta_t}(z\mid z_t)}{q^{\theta,\psi}_{t,\Delta_t}(z\mid z_t)} = \frac{p^{\theta,\star}_{t,\Delta_t}(z\mid z_t)}{p^{\theta,\psi}_{t,\Delta_t}(z\mid z_t)} \cdot \frac{p^{\theta,\psi}_{t,\Delta_t}(z\mid z_t)}{\bar{p}^{\theta,\psi}_{t,\Delta_t}(z\mid z_t)} \cdot \frac{\bar{p}^{\theta,\psi}_{t,\Delta_t}(z\mid z_t)}{q^{\theta,\psi}_{t,\Delta_t}(z\mid z_t)}.
\end{equation}

\paragraph{Step 1: Bounding the twist error.}
We have:
\begin{equation*}
\frac{p^{\theta,\star}_{t,\Delta_t}(z\mid z_t)}{p^{\theta,\psi}_{t,\Delta_t}(z\mid z_t)} = \frac{h^\star_{t+\Delta_t}(z)}{h^\psi_{t+\Delta_t}(z)} \cdot \frac{Z^\psi_t(z_t)}{Z^\star_t(z_t)},
\end{equation*}
where 
\begin{align*}
Z^\psi_t(z_t) &= \sum_{z'} p^\theta_{t,\Delta_t}(z'\mid z_t) h^\psi_{t+\Delta_t}(z') G_{t+\Delta_t}(z')^{\mathbf{1}_{t+\Delta_t\in\{\tau_k\}}},\\
Z^\star_t(z_t) &= \sum_{z'} p^\theta_{t,\Delta_t}(z'\mid z_t) h^\star_{t+\Delta_t}(z') G_{t+\Delta_t}(z')^{\mathbf{1}_{t+\Delta_t\in\{\tau_k\}}}.
\end{align*}

By definition of $\varepsilon_{t+\Delta_t}$, for all $z \in \mathcal{Z}$:
\begin{equation*}
e^{-\varepsilon_{t+\Delta_t}} \leq \frac{h^\psi_{t+\Delta_t}(z)}{h^\star_{t+\Delta_t}(z)} \leq e^{\varepsilon_{t+\Delta_t}}.
\end{equation*}

The normalizing constant ratio can be written as an expectation, by multiplying and dividing each summand in the numerator by $h^\star_{t+\Delta_t}(z')$:
\begin{equation}\label{eq:norm_const}
\frac{Z^\psi_t(z_t)}{Z^\star_t(z_t)} = \sum_{z'} \frac{p^\theta_{t,\Delta_t}(z'\mid z_t) h^\star_{t+\Delta_t}(z') G_{t+\Delta_t}(z')^{\mathbf{1}_{t+\Delta_t\in\{\tau_k\}}}}{Z^\star_t(z_t)} \cdot \frac{h^\psi_{t+\Delta_t}(z')}{h^\star_{t+\Delta_t}(z')} = \E_{p^{\theta,\star}_{t,\Delta_t}(\cdot \mid z_t)}\left[ \frac{h^\psi_{t+\Delta_t}(Z)}{h^\star_{t+\Delta_t}(Z)} \right].
\end{equation}
Since the bounds hold pointwise for each $z'$:
\begin{equation*}
e^{-\varepsilon_{t+\Delta_t}} \leq \frac{Z^\psi_t(z_t)}{Z^\star_t(z_t)} \leq e^{\varepsilon_{t+\Delta_t}}.
\end{equation*}

Combining these bounds:
\begin{equation}\label{eq:twist_bound}
e^{-2\varepsilon_{t+\Delta_t}} \leq \frac{p^{\theta,\star}_{t,\Delta_t}(z\mid z_t)}{p^{\theta,\psi}_{t,\Delta_t}(z\mid z_t)} \leq e^{2\varepsilon_{t+\Delta_t}}.
\end{equation}

\paragraph{Step 2: Bounding the reset error.}
We have:
\begin{equation*}
\frac{p^{\theta,\psi}_{t,\Delta_t}(z\mid z_t)}{\bar{p}^{\theta,\psi}_{t,\Delta_t}(z\mid z_t)} = \frac{h^\psi_{t+\Delta_t}(z)}{h^\psi_{(t+\Delta_t)^-}(z)} \cdot G_{t+\Delta_t}(z)^{\mathbf{1}_{t+\Delta_t\in\{\tau_k\}}} \cdot \frac{\bar{Z}^\psi_t(z_t)}{Z^\psi_t(z_t)},
\end{equation*}
where $\bar{Z}^\psi_t(z_t) = \sum_{z'} p^\theta_{t,\Delta_t}(z'\mid z_t) h^\psi_{(t+\Delta_t)^-}(z')$.

When $t+\Delta_t \notin \{\tau_k\}$, $h^\psi$ is continuous, so $h^\psi_{(t+\Delta_t)^-}(z) = h^\psi_{t+\Delta_t}(z)$ and the ratio equals 1.

When $t+\Delta_t \in \{\tau_k\}$, by definition of $\delta_{t+\Delta_t}$:
\begin{equation*}
|\log h^\psi_{(t+\Delta_t)^-}(z) - \log G_{t+\Delta_t}(z) - \log h^\psi_{t+\Delta_t}(z)| \leq \delta_{t+\Delta_t},
\end{equation*}
which gives:
\begin{equation*}
e^{-\delta_{t+\Delta_t}} \leq \frac{h^\psi_{t+\Delta_t}(z)}{h^\psi_{(t+\Delta_t)^-}(z)} \cdot G_{t+\Delta_t}(z) \leq e^{\delta_{t+\Delta_t}}.
\end{equation*}

For the normalizing constant ratio, we can use the exact same procedure of \eqref{eq:norm_const}:
\begin{equation*}
\frac{\bar{Z}^\psi_t(z_t)}{Z^\psi_t(z_t)} = {\mathbb{E}_{{p}^{\theta,\psi}_{t,\Delta_t}(\cdot\mid z_t)}\left[\frac{h^\psi_{(t+\Delta_t)^-}(Z)}{h^\psi_{t+\Delta_t}(Z)G_{t+\Delta_t}(Z)} \right]}.
\end{equation*}
Since the pointwise bound holds for all $z$, it holds for the expectation:
\begin{equation*}
e^{-\delta_{t+\Delta_t}} \leq \frac{\bar{Z}^\psi_t(z_t)}{Z^\psi_t(z_t)} \leq e^{\delta_{t+\Delta_t}}.
\end{equation*}

Combining:
\begin{equation}\label{eq:reset_bound}
e^{-2\delta_{t+\Delta_t}} \leq \frac{p^{\theta,\psi}_{t,\Delta_t}(z\mid z_t)}{\bar{p}^{\theta,\psi}_{t,\Delta_t}(z\mid z_t)} \leq e^{2\delta_{t+\Delta_t}}.
\end{equation}

\paragraph{Step 3: Bounding the discretization error.}
The distribution $\bar{p}^{\theta,\psi}_{t,\Delta_t}$ corresponds to the transition kernel of a twisted IPS with twist function $h^\psi_{(t+\Delta_t)^-}$. The proposal $q^{\theta,\psi}_{t,\Delta_t}$ is its Euler approximation. Between observation times, $h^\psi$ is Lipschitz continuous, so assumptions (A1)-(A3) hold for the twisted rates. By Lemma \ref{lem:chisq_bound}:
\begin{equation}\label{eq:disc_bound}
\mathbb{E}_{q^{\theta,\psi}(\cdot\midd z_t)}\left[\left(\frac{\bar{p}^{\theta,\psi}_{t,\Delta_t}(Z\midd z_t)}{q^{\theta,\psi}_{t,\Delta_t}(Z\midd z_t)}\right)^2\right] = 1 + \chi^2(\bar{p}^{\theta,\psi}_{t,\Delta_t} \| q^{\theta,\psi}_{t,\Delta_t}) \leq 1 + C\Delta_t^2.
\end{equation}

\paragraph{Step 4: Combining the bounds.}
From \eqref{eq:iw_decomp}, squaring and taking expectations:
\begin{align*}
\mathbb{E}_{q^{\theta,\psi}(\cdot\midd z_t)}\left[\left(\frac{p^{\theta,\star}_{t,\Delta_t}(Z\midd z_t)}{q^{\theta,\psi}_{t,\Delta_t}(Z\midd z_t)}\right)^2\right] &= \mathbb{E}_{q^{\theta,\psi}(\cdot\midd z_t)}\left[\left(\frac{p^{\theta,\star}(Z\midd z_t)}{p^{\theta,\psi}(Z\midd z_t)}\right)^2 \left(\frac{p^{\theta,\psi}(Z\midd z_t)}{\bar{p}^{\theta,\psi}(Z\midd z_t)}\right)^2 \left(\frac{\bar{p}^{\theta,\psi}(Z\midd z_t)}{q^{\theta,\psi}(Z\midd z_t)}\right)^2\right] \\
&\leq e^{4\varepsilon_{t+\Delta_t}} \cdot e^{4\delta_{t+\Delta_t}} \cdot (1 + C\Delta_t^2) \\
&= e^{4\varepsilon_{t+\Delta_t} + 4\delta_{t+\Delta_t}} (1 + C\Delta_t^2),
\end{align*}
where we used that the first two ratios are bounded uniformly by \eqref{eq:twist_bound} and \eqref{eq:reset_bound}.

Therefore:
\begin{equation*}
\mathrm{ESS}_{t,\Delta_t}(z_t) = \mathbb{E}_{q^{\theta,\psi}(\cdot\midd z_t)}\left[\left(\frac{p^{\theta,\star}_{t,\Delta_t}(Z\midd z_t)}{q^{\theta,\psi}_{t,\Delta_t}(Z\midd z_t)}\right)^2\right]^{-1} \geq \frac{\exp({-4(\varepsilon_{t+\Delta_t} +\delta_{t+\Delta_t}))}}{1 + C\Delta_t^2}.
\end{equation*}

This bound holds uniformly in $z_t$ since $\varepsilon_{t+\Delta_t}$ and $\delta_{t+\Delta_t}$ are defined as suprema over all states, and the constant from Lemma \ref{lem:chisq_bound} is uniform.

\subsection{Wake-sleep objective}\label{proof:ws}

We discuss here the objectives used in Section~\ref{sec:twist_learning}. Fix an observed sequence $y_{1:K}$, and let $P^\star_\theta(\cdot \mid y_{1:K})$ denote the posterior path measure from \eqref{eq:posterior_path_measure}. We write $P^\psi_\theta(\cdot \mid y_{1:K})$ for the approximate posterior IPS with initial distribution $q^\psi_0(\cdot \mid y_{1:K})$ and local rates
\[
r^{\theta,\psi}_{i,t}(v\mid z) = r^\theta_{i,t}(v\mid z)\,s^\psi_{i,t}(v,z),
\qquad
s^\psi_{i,t}(v,z) \coloneq \frac{h^\psi_t(z^{i\to v})}{h^\psi_t(z)}.
\]

\paragraph{Continuous-time forward KL.}
Let $z_{[0,T]}$ be a càdlàg trajectory with jump times $0<u_1<\cdots<u_N<T$, and let $i_n$ denote the coordinate that changes at time $u_n$. Using the standard density of a CTMC path measure, the log-density of $P^\psi_\theta(\cdot \mid y_{1:K})$ is
\begin{align}
\log P^\psi_\theta(z_{[0,T]}\mid y_{1:K})
=
\log q^\psi_0(z_0\mid y_{1:K})
+
\sum_{n=1}^N
\log r^{\theta,\psi}_{i_n,u_n}(z^{i_n}_{u_n}\mid z_{u_n^-})
+
\int_0^T \sum_{i\in\gI} r^{\theta,\psi}_{i,t}(z^i_t\mid z_t)\,dt,
\label{eq:log_path_q}
\end{align}
where, as in the main text, $r^{\theta,\psi}_{i,t}(z^i\mid z)$ denotes the diagonal entry $r^{\theta,\psi}_{i,t}(z^i\mid z) = -\sum_{v\neq z^i} r^{\theta,\psi}_{i,t}(v\mid z)$.
We study the objective
\begin{align}
D_{\mathrm{KL}}\!\left(P^\star_\theta(\cdot\mid y_{1:K})\,\middle\|\,P^\psi_\theta(\cdot\mid y_{1:K})\right)
&=
C(\theta,y_{1:K})
-
\mathbb{E}_{P^\star_\theta}\!\left[\log P^\psi_\theta(Z_{[0,T]}\mid y_{1:K})\right],
\label{eq:kl_expand_1}
\end{align}
where $C(\theta,y_{1:K})$ collects all terms independent of $\psi$. 
Using
\[
\log r^{\theta,\psi}_{i,u}(Z^i_u\mid Z_{u^-})
=
\log r^\theta_{i,u}(Z^i_u\mid Z_{u^-})
+
\log s^\psi_{i,u}(Z^i_u,Z_{u^-}),
\]
we can absorb the $\log r^\theta_{i,u}$ term into the constant when optimizing with respect to $\psi$, yielding the continuous-time sleep objective
\begin{align}
\mathcal{L}^{\mathrm{cont}}_{\mathrm{s}}(\psi; y_{1:K},\theta)
&\coloneq
-\mathbb{E}_{P^\star_\theta}\!\Bigg[
\log q^\psi_0(Z_0\mid y_{1:K})
+
\int_0^T \sum_{i\in\gI} r^{\theta,\psi}_{i,t}(Z^i_t\mid Z_t)\,dt+
\sum_{i\in\gI}\sum_{u:Z^i_u\neq Z^i_{u^-}}
\log s^\psi_{i,u}(Z^i_u,Z_{u^-})
\Bigg].
\label{eq:cont_sleep_obj}
\end{align}
Equivalently, up to an additive constant and an overall sign convention, this is the quantity reported in \eqref{eq:fwd_kl}.

\paragraph{Time discretization.}
Let $0=t_0<t_1<\cdots<t_M=T$ be a grid containing $\{\tau_k\}_{k=1}^K$, and write $\Delta_{t_m}=t_m-t_{m-1}$. Using a Riemann approximation for the integral term in \eqref{eq:cont_sleep_obj}, and replacing the jump sum by the corresponding coordinate changes on the grid, we obtain
\begin{align}
\mathcal{L}_{\mathrm{s}}(\psi;\,z_{t_0:t_M},y_{1:K},\theta)
\!=\!
-\log q^\psi_0(z_0\mid y_{1:K})\!
-\!\!
\sum_{m=0}^{M-1}\sum_{i\in\gI}
\Bigl[
\Delta_{t_{m+1}}\,r^{\theta,\psi}_{i,t_m}(z^i_{t_m}\mid z_{t_m})
+
\boldsymbol{1}[z^i_{t_m}\neq z^i_{t_{m+1}}]
\log s^\psi_{i,t_m}(z^i_{t_{m+1}},z_{t_m})
\Bigr],
\label{eq:appendix_sleep_disc}
\end{align}
which is precisely the objective in \eqref{eq:approx_fwd_kl}. In practice, we estimate the expectation of \eqref{eq:appendix_sleep_disc} under the joint law of $(Z_{[0,T]},Y_{1:K})$ by ancestral sampling from the prior dynamics and emission model.

\paragraph{Wake phase.}
For $\theta$, we optimize the marginal likelihood \eqref{eq:latent_ips_likelihood}. By Fisher's identity,
\begin{align}
\nabla_\theta \mathcal{L}(\theta)
=
\mathbb{E}_{P^\star_\theta(\cdot\mid y_{1:K})}\!\left[
\sum_{k=1}^K \nabla_\theta \log G_{\tau_k,\theta}(Z_{\tau_k})
+
\nabla_\theta \log P_\theta(Z_{[0,T]})
\right].
\label{eq:fisher_appendix}
\end{align}
To obtain a tractable estimator, we discretize the prior path density using the same Euler grid as above. The resulting approximation is
\begin{align}
\log \hat P_\theta(z_{[0,T]})
&=
\log p^\theta_0(z_0)
+
\sum_{m=0}^{M-1}\sum_{i\in\gI}
\Bigl[
\Delta_{t_{m+1}}\,r^\theta_{i,t_m}(z^i_{t_m}\mid z_{t_m})
+
\boldsymbol{1}[z^i_{t_m}\neq z^i_{t_{m+1}}]
\log r^\theta_{i,t_m}(z^i_{t_{m+1}}\mid z_{t_m})
\Bigr],
\label{eq:appendix_disc_path}\end{align}
which matches \eqref{eq:disc_path}. Substituting \eqref{eq:appendix_disc_path} into \eqref{eq:fisher_appendix}, and approximating the posterior expectation by weighted trajectories produced by tSMC, gives the wake estimator
\begin{align}
\widehat{\nabla_\theta \mathcal{L}_{\mathrm{w}}}
\big(\theta;\,\{z^{(s)}_{t_0:t_M},\bar w_T^{(s)}\}_{s=1}^S,y_{1:K}\big)
=
\sum_{s=1}^S \bar w_T^{(s)}
\left[
\sum_{k=1}^K \nabla_\theta \log G_{\tau_k,\theta}(z^{(s)}_{\tau_k})
+
\nabla_\theta \log \hat P_\theta(z^{(s)}_{[0,T]})
\right],
\label{eq:appendix_wake_estimator}
\end{align}
which is the estimator reported in \eqref{eq:approx_mle}. Observe that, when optimizing $\theta$, only the emission terms and the prior path density contribute to the gradient; the twist affects the wake phase only through the quality of the particle approximation to the posterior trajectories.

\section{Algorithm}\label{appdx:algo}
We provide detailed pseudocode of our training scheme using TwistNet and the KL loss in Algorithm \ref{algo:wake_sleep_detailed}. Note that Algorithm \ref{algo:wake_sleep_detailed} runs tSMC as a subroutine, which we describe in detail in Algorithm \ref{algo:tsmc_short}. 

Let the cost of a forward pass of the rates, the score, and the potentials be $C_r,C_s,C_G$, respectively. Then, the time complexity of an update for a single datapoint in the sleep phase is
\begin{equation*}
    \Theta\big(\ \underbrace{M C_r +K C_G}_{\text{Simulation}}\ +\   \underbrace{M C_s}_{\text{Loss}}\ \big),
\end{equation*}
while in the wake phase it is
\begin{equation*}
    \Theta\big(\ \underbrace{\overbrace{S(M (C_r+C_s) +K C_G)}^{\text{Simulation}} +\!\!\!\!\!   \overbrace{M S }^{\text{Resampling}} \!\!\!\!}_{\text{tSMC}}\ +\   \underbrace{M C_r }_{\text{Loss}}\ \big),
\end{equation*}
where the cost of computing the twist is absorbed into that of computing the score, since we are using the efficient parameterization described in Section \ref{sec:efficient}. Using this parameterization, rather than having to compute $d\,(|\gV| -1) + 1$ forward passes of a twist model, we have a cost that is roughly $\Theta(C_\Phi + (d\,(|\gV| -1) + 1)C_\rho)$, where $C_\Phi$ and $C_\rho$ are the cost of the context encoder and the aggregator in \eqref{eq:rho}. In our experiments, we let $\Phi$ bear the cost of heavy operations such as processing future observations, observation times, as well as covariates and positional information, while $\rho$ is a simple two-layer MLP.

The time and memory cost of the loss terms can be reduced to $\Theta(C_s)$ and $\Theta(C_r )$ by employing a Monte Carlo approximation of time, only considering a single timestep for each update. This is particularly useful when large neural networks are employed in parameterizing either the rates or the score.

A few practical notes:
\begin{enumerate}[label=(\roman*)]
\item To amortize simulation cost, after sampling a mini-batch of trajectories (sleep) or running tSMC (wake) we perform several optimizer steps on the \emph{same} batch before regenerating trajectories.

\item In our experiments we resample at every step, hence particle weights are $1/S$. Due to compute constraints we use relatively few particles (dictated by high dimensionality), and often observe particle collapse. To reduce memory cost, in the wake phase we do not weight the loss across particles. Instead, we draw a single path $z^{(s^\star)}_{[t_0:t_M]}$ by importance resampling using the final normalized weights and compute the wake loss on that path only. This yields a consistent estimator under importance resampling and is effectively equivalent to using self-normalized importance weights \citep{chopin_introduction_2020}.
\item Resampling is performed using systematic resampling \citep{chopin_introduction_2020}.

\end{enumerate}

\begin{algorithm}[!t]
\small
\caption{Wake--Sleep with twisted SMC for latent IPSs}
\label{algo:wake_sleep_detailed}
\begin{algorithmic}[1]
\STATE \textbf{Inputs:} $\mathcal D_{\mathrm{train}}=\{(y^{(b)}_{1:K},\tau^{(b)}_{1:K})\}$; optimizer steps $\textsc{GradStep}_\psi, \textsc{GradStep}_\theta$; time grid $0=t_0<\dots<t_M=T$; \# of particles $S$; batch size $B$; global updates $G$; updates per phase $N$; initializations $\psi_0,\,\theta_0$; Monte Carlo loss flag $\texttt{mc\_loss}$
\STATE $\psi \leftarrow \psi_0, \theta \leftarrow \theta_0$.
\FOR{$g=1,\dots,G$} 
  \STATE \textcolor{gray}{\textit{{\# Sleep phase}}}
  \FOR{$n=1,\dots,N$}
    \FOR{$b=1,\dots,B$}
    \STATE Simulate $z^{(b)}_{[t_0:t_M]}$ from the prior $P_\theta$ via Euler steps.
    \STATE Simulate synthetic observations $\tilde y^{(b)}_k\sim p_\theta(\cdot\mid z^{(b)}_{\tau_k})$, for $k=1,\dots,K$.
    \IF{$\texttt{mc\_loss}$}
    \STATE $m\sim \mathcal{U}(\{0,\dots,M-1\})$
    \STATE $\displaystyle 
      \ell_{\mathrm{sleep}}^{(b)}
=
-\log q^\psi_0(z_0^{(b)}\mid \tilde y^{(b)}_{1:K})
-
M\sum_{i\in\gI}
\Big[
\Delta_{t_{m+1}}\,r^{\theta,\psi}_{i,t_m}(z^{i,(b)}_{t_m}\mid z^{(b)}_{t_m})
+
\mathbf 1[z^{i,(b)}_{t_m}\neq z^{i,(b)}_{t_{m+1}}]
\log s^\psi_{i,t_m}(z^{i,(b)}_{t_{m+1}},z^{(b)}_{t_m})
\Big]
      $
    \ELSE
    \STATE $\displaystyle 
      \ell_{\mathrm{sleep}}^{(b)}(\psi)
=
-\log q^\psi_0(z_0^{(b)}\mid \tilde y^{(b)}_{1:K})
-
\sum_{m=0}^{M-1}\sum_{i\in\gI}
\Big[
\Delta_{t_{m+1}}\,r^{\theta,\psi}_{i,t_m}(z^{i,(b)}_{t_m}\mid z^{(b)}_{t_m})
+
\mathbf 1[z^{i,(b)}_{t_m}\neq z^{i,(b)}_{t_{m+1}}]
\log s^\psi_{i,t_m}(z^{i,(b)}_{t_{m+1}},z^{(b)}_{t_m})
\Big].$
      \ENDIF
    \ENDFOR
    \STATE $\psi \leftarrow \textsc{GradStep}_\psi\left(\nabla_\psi \frac{1}{B}\sum_{b=1}^B \ell_{\mathrm{sleep}}^{(b)}(\psi)\right)$
  \ENDFOR
  \STATE $\bar \theta \leftarrow \theta$ \textcolor{gray}{\textit{{\# Lagged $\theta$ to use for tSMC proposal}}}
    \STATE \textcolor{gray}{\textit{{\# Wake phase}}}
  \FOR{$n=1,\dots,N$}
    \FOR{$b=1,\dots,B$}
      \STATE $(y^{(b)}_{1:K},\tau^{(b)}_{1:K}) \sim \mathcal D_{\mathrm{train}}$
      \STATE Simulate approx. posterior via $\{\!z^{(b,s)}_{[t_0:t_M]},\,\bar w^{(b,s)}\!\}_{s=1}^S \leftarrow \text{tSMC}\left( y^{(b)}_{1:K},\tau^{(b)}_{1:K}\right)$ with prior dynamics $P_\theta$, potentials $G_{\cdot,\theta}$, twist function $h^\psi$ and proposal $q^{\bar\theta,\psi}$ from \eqref{eq:twist_approx_kernel} (see Algorithm \ref{algo:tsmc_short}).
      \STATE $z^{(b)}_{[t_0:t_M]}\leftarrow z^{(b, s^\star)}_{[t_0:t_M]}$, where $s^\star\sim \text{Categorical}\left(\left\{ \bar w^{(b,s)}\right\}^{S}_{s=1}\right)$
    \IF{$\texttt{mc\_loss}$}
    \STATE $m\sim \mathcal{U}(\{0,\dots,M-1\}), \,k\sim \mathcal{U}(\{1,\dots,K\})$
    \STATE $\displaystyle 
      \ell_{\mathrm{wake}}^{(b)}(\theta)
=
-\log p^\theta_0(z_0^{(b)})
-
M\sum_{i\in\gI}
\Big[
\Delta_{t_{m+1}}\,r^\theta_{i,t_m}(z^{i,(b)}_{t_m}\mid z^{(b)}_{t_m})
+
\mathbf 1[z^{i,(b)}_{t_m}\neq z^{i,(b)}_{t_{m+1}}]
\log r^\theta_{i,t_m}(z^{i,(b)}_{t_{m+1}}\mid z^{(b)}_{t_m})
\Big]
-
K \log G_{\tau_k,\theta}(z^{(b)}_{\tau_k}).$
    \ELSE
    \STATE $\displaystyle 
      \ell_{\mathrm{wake}}^{(b)}(\theta)
=
-\log p^\theta_0(z_0^{(b)})
-
\sum_{m=0}^{M-1}\sum_{i\in\gI}
\Big[
\Delta_{t_{m+1}}\,r^\theta_{i,t_m}(z^{i,(b)}_{t_m}\mid z^{(b)}_{t_m})
+
\mathbf 1[z^{i,(b)}_{t_m}\neq z^{i,(b)}_{t_{m+1}}]
\log r^\theta_{i,t_m}(z^{i,(b)}_{t_{m+1}}\mid z^{(b)}_{t_m})
\Big]
-
\sum_{k=1}^K \log G_{\tau_k,\theta}(z^{(b)}_{\tau_k}).$
      \ENDIF
    \ENDFOR
    \STATE $\theta \leftarrow \textsc{GradStep}_\theta\!\left(\nabla_\theta \frac{1}{B}\sum_{b=1}^B \ell_{\mathrm{wake}}^{(b)}(\theta)\right)$
  \ENDFOR
\ENDFOR
\end{algorithmic}
\end{algorithm}

\begin{algorithm}[t]
\small
\caption{tSMC}
\label{algo:tsmc_short}
\begin{algorithmic}[1]
\STATE \textbf{Inputs:} $(y_{1:K},\tau_{1:K})$; time grid $0=t_0<\dots<t_M=T$; \# of particles $S$; prior kernel $p^{\theta}$ and initial distribution $p_0^\theta$; potential function $G_{\cdot,\theta}$; twist function $h^\psi$; proposal kernel $q^{\bar\theta,\psi}$ and initial distribution $q^\psi_0$; ESS threshold $\tau_{\mathrm{ESS}}$
\STATE \textbf{Output:} $\{z^{(s)}_{[t_0:t_M]},\,\bar w^{(s)}\}_{s=1}^S$
\STATE Initialize 
\begin{equation*}
    z^{(s)}_{t_0}\sim q^\psi_0,\quad w^{(s)}_0 \leftarrow \frac{p_0^\theta(z_0^{(s)})h^\psi_0(z_0^{(s)})G_{0,\theta}(z_0^{(s)})^{\boldsymbol{1}_{\{\tau_k\}}(0)}}{q^\psi_0(z_0^{(s)})}, \quad \text{for} \quad s=1,\dots,S.
\end{equation*}
\FOR{$m=0,\dots,M-1$}
\FOR{$s=1,\dots,S$}
\IF{$\mathrm{ESS}(\{ w^{(s)}_{m}\}) < \tau_{\mathrm{ESS}}$}
\STATE Resample $z_{t_0:t_m}^{(s)} \leftarrow z_{t_0:t_m}^{(a_s)}$ using ancestor $a_s\sim \text{Categorical}\left( \left\{ w^{(s)}_{m}\big/\sum_{u=1}^S w^{(u)}_{m}\right\} \right)$.
\STATE Reset weights to $w^{(s)}_{m} \leftarrow 1$
\ENDIF
\ENDFOR
\FOR{$s=1,\dots,S$}
  \STATE Propose $z^{(s)}_{t_{m+1}} \sim q^{\bar\theta,\psi}_{t_m,\Delta t}(\cdot\mid z^{(s)}_{t_m})$
  \STATE Update weights 
  \begin{equation*}
  w^{(s)}_{m+1} \leftarrow w^{(s)}_{m}\times \frac{p^{\theta}_{t_m,\Delta t}(z^{(s)}_{t_{m+1}}\mid z^{(s)}_{t_m})}{q^{\bar\theta,\psi}_{t_m,\Delta t}(z^{(s)}_{t_{m+1}}\mid z^{(s)}_{t_m})}    \times \frac{h^\psi_{t_{m+1}}(z^{(s)}_{t_{m+1}})}{h^\psi_{t_{m}}(z^{(s)}_{t_m})} \times G_{t_{m+1},\theta}(z_{t_{m+1}}^{(s)})^{\boldsymbol{1}_{\{\tau_k\}}(t_{m+1})}
  \end{equation*}
  \ENDFOR
\ENDFOR
\STATE \textbf{return} paths $\{z^{(s)}_{[t_0:t_M]}\}$ and final normalized weights $\{ w^{(s)}_{M}\big/\sum_{u=1}^S w^{(u)}_{M}\}$
\end{algorithmic}
\end{algorithm}

\section{Experimental details}\label{appdx:exp_details}
All code for the experiments can be found at the following link: \href{https://github.com/giosueio/LatentIPS}{\textcolor{teal}{\underline{https://github.com/giosueio/LatentIPS}}}.
\subsection{Baselines}\label{base:details}
We compare our method against alternatives that are scalable to high-dimensional systems, amenable to gradient-based optimization of neural models for the dynamics of the model, and do not require expensive operations such as backpropagation through time. This rules out most approaches based on a reverse KL objective, such as \cite{naesseth_variational_2017, lawson2022sixo}. We note that the NeuralMJP method in \cite{seifner2023neural} could, in principle, ease the memory cost of backpropagating through time by using the adjoint method \citep{chen2018neural}. However, we failed to have this method converge to meaningful solutions in our setup, due to exploding gradients -- possibly coupled with those gradients being biased, due to the repeated Gumbel-softmax approximations \citep{seifner2023neural, jang2017categorical}. We note that in \cite{seifner2023neural} high-dimensional, interacting systems such as the ones considered in this work were not addressed. We also exclude methods that are not amenable to amortization, as this would make inference for neural models much more complex. This rules out the methods proposed in \cite{opper2007variational, wildner2019moment, kohs2021variational, eichentropic}.

The baselines we consider are:
\begin{itemize}
    \item \textbf{Bootstrap particle filter} (BPF)\citep{doucet2009tutorial}: this method is an SMC algorithm using the filtering distributions as intermediate targets, and the prior transition probabilities as a proposal distribution.
    \item \textbf{NASMC and NAS-X }\citep{gu2015neural,lawson2023x}: these are both SMC algorithms with an informative proposal learned with a forward KL loss. In \cite{gu2015neural} the intermediate targets are the filtering distributions, while in \cite{lawson2023x} the intermediate targets are the twisted distributions, where the twist function is learned using a density ratio estimation loss. We parameterize the proposal distribution by fitting a score network minimizing the forward KL loss in \ref{eq:approx_fwd_kl},
    \begin{equation*}
        \log\frac{h_t(z^{i\rightarrow v})}{h_t(z)} \approx s^\psi_t(z)_{[i,v]},
    \end{equation*}
    where $s^\psi_t: \gZ \rightarrow \R^{d\times V} $ and $s^\psi_t(z)_{[i,z^i]} = 1$ for $i \in \gI$.
    
    \item \textbf{Taylor-approximated guidance} (\textsc{TAG}):
to overcome the inefficiency of computing $d\times(V-1)+1$ forward passes of the twist function, \cite{nisonoff_unlocking_2025} proposed to compute a first-order Taylor approximation of the log-twist evaluated at a specific value $z$, i.e. 
\begin{equation}\label{eq:tag}
    \log h^\psi_t(z_t)\approx \log h^\psi_t(\mathbf{z}) + \mathbf{z}_t^\top \nabla_\mathbf{z}\log h^\psi_t(\mathbf{z})
\end{equation}
where $\mathbf{z},\mathbf{z}_t$ are one-hot encoded versions of $z, z_t$, enabling a single forward pass
at $z$ of the twist function. Note that backpropagation of the KL loss with respect to the score function implied by \eqref{eq:tag} would require computing a second derivative, on top of the spatial derivative with respect to $\mathbf{z}$. This can be extremely expensive for neural models, therefore we only consider a DRE loss for this method.
\end{itemize}
For all of our methods (except BPF), we also train a variational initial distribution $q_0^\psi(z_0)$ with a forward KL loss on generated trajectories (see the full objective in \eqref{eq:approx_fwd_kl}).

The density ratio estimation (DRE) loss \citep{lawson2022sixo, lawson2023x} used to learn the twist in NAS-X and TAG is
\begin{equation}
    \hat{\mathcal{L}}_\text{DRE}(\psi) = \sum_{t\in \mathcal{T}} \sum_{i\in \gI} \log \sigma(\log h_t^\psi(z^+_t;y_{\geq t},\tau_{\geq t})) + \log(1-\sigma(\log h_t^\psi(z^-_t;y_{\geq t},\tau_{\geq t}))),
\end{equation}
where $\sigma:\R\rightarrow [0,1]$ is the logistic function. Positive samples $z^+_t$ are generated by the forward model using ancestral sampling of $z^+_{[0,T]}\sim P_\theta$ first, and then $y_{1:K}, \tau_{1:K}$. Negative samples $z^-_t \sim P_\theta$, and are hence uncoupled from $y_{1:K}, \tau_{1:K}$. Using this loss is equivalent to training a classifier to distinguish between coupled and uncoupled samples. 

All of the methods were trained on a single NVIDIA RTX A5000 with 24GB of VRAM.

\subsection{SIRS model}\label{sirs:details}
\paragraph{Data generation.}
\label{appdx:data}
We simulate SIRS epidemics on undirected graphs with $d$ nodes. Graphs are sampled from an expected-degree model using \texttt{networkx.expected\_degree\_graph} \citep{SciPyProceedings_11}, and all nodes have expected degree 5. Each node $i$ has a feature vector $\xi_i\in\R^{16}$, included to make posterior inference more challenging. Ground truth paths on $[0,T]=[0,10]$ are drawn using Gillespie’s algorithm \citep{gillespie1977exact, wilkinson2018stochastic}, with rate parameters in \eqref{eq:sirs_rates} fixed to $(\alpha_0,\alpha_1,\beta,\gamma)=(0.1,\,1.0,\,0.4,\,0.05)$. We assign 50 trajectories for the training set, and 50 to the test set.

\paragraph{Observations model.}
\label{appdx:obs_target} Observations are sampled at $K=10$ snapshot times $\tau_1<\dots<\tau_K$, sampled from a uniform distribution in $[0,10]$ for each trajectory. Observations are conditionally independent across nodes, given $Z_{\tau_k}$. For each node $i\in\gI$ we include an explicit mask token $\varnothing$ and use the node-factorized emission distribution
\[
  p\!\left(y_{\tau_k}\mid Z_{\tau_k}=z\right)
  \;=\; \prod_{i\in\gI} g\!\left(y^i_{\tau_k}\mid z^i\right),
\]
with a masking probability set to $p_{\mathrm{mask}}=\frac{1}{2}$ and small symmetric label noise $\delta>0$, for numerical stability:
\[
  g(c\mid z^i)\;=\;
  \begin{cases}
    p_{\mathrm{mask}}, & c=\varnothing,\\[4pt]
    (1-p_{\mathrm{mask}})\!\left[(1-\delta(V-1))\,\mathbf{1}\{c=z^i\}+\delta\,\mathbf{1}\{c\neq z^i\}\right], & c\in\gV,
  \end{cases}
\]
where $V=\abs{\gV}$. 

\paragraph{Twist model.} We parameterize the score networks and the encoder of the TwistNet with a Graph Transformer (GT) from \cite{vignac2022digress}. For all of our models we use 2 GT layers and 4 heads, with a node embedding dimension of 64, an edge embedding dimension of 8, and global information embedding of 32, for a total of 485,827 parameters. We did not tune any of these hyperparameters and we keep them fixed throughout our experiments. For any time $t\in[0,T]$, we feed as input the feature vector, future observation $y_{\geq t}$ and observation times $\tau_{\geq t}$, and graph statistics computed on the adjacency matrix as in \citep{vignac2022digress}. The score network also takes as input the current state $z_t$, while in the TwistNet this is only considered when aggregating encoder outputs. For the TwistNet, we let the last layer be a two-layer MLP with $m=64$. 
\begin{wrapfigure}{!tr}{0.44\textwidth}
    \includegraphics[width=0.99\linewidth]{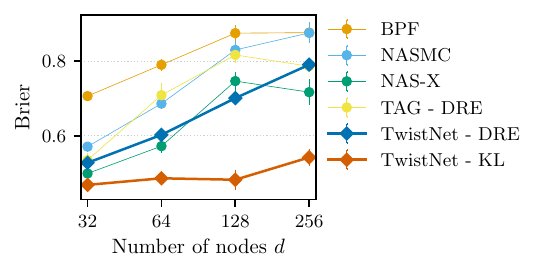}
    \caption{Latent trajectory reconstruction of each inference method, measured by Brier score of true latent trajectories with respect to the posterior approximations. Methods using TwistNet are highlighted by a thicker line. Error bars correspond to two standard errors.}
    \label{fig:brier}
\end{wrapfigure}
\paragraph{Training.} For optimization of the TwistNet and score models, we use the Adam optimizer \citep{kingma2014adam} with learning rate $0.001$ in the latent trajectory inference task, and $0.0003$ for parameter learning. Parameters of the rates in \eqref{eq:sirs_rates} are optimized using Adam with learning rate $0.005$.
All of the other hyperparameters are set to their default in PyTorch \footnote{\texttt{https://docs.pytorch.org/docs/stable/generated/torch.optim.Adam.html}}. We did not tune these values.
For our latent trajectory inference experiment, we trained using the forward KL loss on the twist and score models for 1,000 steps, with a batch size of 32 and $\Delta_t=0.1$. We note that all of the methods converged. Covariates $\xi$ to sample from the prior in the sleep phase were sampled from the test set. For the parameter learning task, we initialize all parameters of the rates at $0.2$. We then train with batch size $16$, using $\Delta_t=0.05$, $S=10$ tSMC particles for each sample in the wake phase, and we resample at every step (i.e., $\tau_\text{ESS}=1$ in Algorithm \ref{algo:tsmc_short}). We found it helpful to begin the training loop by optimizing $\psi$ with sleep steps (see Algorithm \ref{algo:wake_sleep_detailed}) until convergence, which could be achieved in roughly 2,500 steps, and then alternating wake and sleep steps. We reuse each simulated mini-batch for multiple inner updates to reduce simulator calls, using $25$ optimizer steps per batch before resampling fresh trajectories and employing a Monte Carlo approximation of the time summation in the objective. We use 25 steps in each phase for 25 global iterations, i.e. for a total of 625 steps (not counting the repeated steps for each batch). In the wake phase we compute the loss on a \emph{single} trajectory sampled by importance resampling from the final particle weights, rather than a weight-averaged objective, as explained in Section \ref{appdx:algo}.

\paragraph{Evaluation.} In our \textit{latent trajectory inference} experiment we are interested in understanding whether our method can be used to perform posterior inference given a set of observations and a prescribed forward model, with no parameter learning. From tSMC particles we form per-time nodewise marginals $\hat p_t(z) = \frac{1}{S}\sum_{s=1}^S\boldsymbol{1}(z_{t}^{(s)} =z)$, where weights are uniform because we resampled at every step (i.e. we let $\tau_\text{ESS}=1.$ in Algorithm \ref{algo:tsmc_short}).
For each method we consider 25 particles ($S=25$), except for BPF for which we take 250 particles.
We also add a small uniform weight to avoid numerical issues when support is scarce:
\[
  \tilde p_t \;=\; (1-\epsilon)\,\hat p_t \;+\; \epsilon \,\mathrm{Unif}(\gV).
\]
Let $z^\star_t\in\gZ$ be the ground-truth state of the latent trajectory at time $t \in [0,T]$. We report an average of the following metrics over a test set of 50 trajectories:
\[
  \text{CE} \!= \!-\frac{1}{M}\!\sum_{m=1}^M \log \tilde p_{t_m}(z^\star_{t_m}),\quad\!
  \text{Brier} \!=\! \frac{1}{M}\!\sum_{m=1}^M \|\tilde p_{t_m} - \mathbf{z}_{t_m}^\star \|_2^2.
\]
where $0=t_1<\dots<t_M=T$ is the set of discretized time indices, and $\mathbf{z}_t^\star$ is the one-hot encoded $z^\star_t$. The CE loss over dimensions is displayed in Figure \ref{fig:nll_vs_dim}, and the Brier score in Figure \ref{fig:brier}. 

In our \textit{parameter learning} experiment, we keep track of individual parameter estimates and total relative parameter error $\sum_j{|\hat \theta_j- \theta_j|}/{|\theta_j|}$ over the four parameters in equation \ref{eq:sirs_rates}.

\subsection{Wildfires trajectories}\label{wf:details}
\paragraph{Dataset.} As mentioned in Section \ref{sec:wildfires}, we consider a subset of the trajectories in the WildfireSpreadTS dataset \citep{gerard2023wildfirespreadts}. We filter them based on the following criteria:
\begin{itemize}
    \item We take a $64\times 64$ crop at the center of the image at the day corresponding to the starting date in GlobFire \citep{artes2019global}, and consider trajectories that are a single week long. Note that images are created based on the final shape of the fire \citep{gerard2023wildfirespreadts}, hence the center point does not necessarily correspond to the point of ignition.
    \item Since many of these crops do not contain any active fire pixel, we filter trajectories so that the $64\times 64$ grid at the first day contains at least one fire pixel, and there are at least two of the other six days with at least one pixel of active fire.
\end{itemize}
All covariates except for VIIRS reflectance channels are excluded. We decided to keep these variables because \cite{gerard2023wildfirespreadts} showed empirically that they are the most predictive, in terms of cross-validation error.

\paragraph{Rates model.} Local rates are computed via a neural network that takes as input the VIIRS channels $x_{1:7}\in \R^{64\times 64\times 3\times 7}$ for the entire week (i.e., we are performing prediction conditioning on covariates that occur at future time points) and the current one-hot encoded state $\mathbf{z}_t\in\{0,1\}^{64\times 64\times 3}$, concatenated to VIIRS inputs along the channel dimension after repeating 7 times along the time dimension. The model is a UTAE network \citep{garnot2021panoptic}, selected for its good performance in \cite{gerard2023wildfirespreadts}. This model resembles a UNET, with temporal self-attention \citep{garnot2021panoptic}. We set the number of channels at each layer to be $[64, 64, 64, 128]$ for upsampling steps and $[128,64,32,32,3]$ for downsampling steps, and the embedding dimension for self-attention to 128, for a total of $1,574,725$ parameters. We did not tune these hyperparameters. We let the output be a field in $\mathbb{R}^{64\times 64\times 3}$, and exponentiate to ensure positivity. Before turning these into outflow rates, we mask values so that transitions from unburned ($U$) to burned ($B$) have rates zero, and set the rates of transitioning to active ($A$) to zero for all pixels with no Moore neighbors being active. Finally, we fill the diagonal of the outflow rates with the negative sum of the off-diagonals. We let the initial distribution be a UTAE with channels $[32, 32, 64]$ for upsampling and $[64,32,32]$ for downsampling, conditioned on VIIRS channels.

\paragraph{Observations model.} Observations are binarized fields of active fire pixels, taken directly from \citep{gerard2023wildfirespreadts}. These are obtained once a day by aggregating two overpasses from the VIIRS satellite, hence are not exactly an instantaneous snapshot. For simplicity, we treat it as such and model the emission distribution using a simple scalar value: $p\!\left(y^i_{k}=1\mid Z^i_{\tau_k} =A\right)=\sigma(\theta_\text{detect})$, where $\theta_\text{detect}$ is a scalar logit and $\sigma$ the logistic function. We set $p\!\left(y^i_{k}=1\mid Z^i_{\tau_k} \in \{ U,B\}\right)=\delta$ with $\delta = 0.001$ for numerical stability. In other words, if a fire is active in a single pixel we have probability $\sigma(\theta_\text{detect})$ of observing it, and if a fire is not active we will most likely not detect it (probability of $1-\delta$ of the observed pixel being zero). This choice is scientifically motivated: pixels are likely to be obscured by clouds and smoke, and we assume contamination does not occur in the other direction (i.e. very few false negatives).
\begin{figure}[t]
\centering
\footnotesize
\captionof{table}{Parameter estimates and relative parameter error (RPE) for the SIRS model with 64 nodes, mean $\pm$ 2 standard deviations across 10 random seeds.}
\label{tab:final-param-estimates-64}
\begin{tabular}{lccccc}
 & $\boldsymbol{\alpha_0}$ & $\boldsymbol{\alpha_1}$ & $\boldsymbol{\beta}$ & $\boldsymbol{\gamma}$ & \textbf{RPE} \\
\hline \\
\textbf{Ground truth} & 0.1 & 1.0 & 0.4 & 0.05 & -- \\
TwistNet - KL & 0.117 $\pm$ 0.021 & \textbf{0.888} $\pm$ 0.115 & \textbf{0.396} $\pm$ 0.030 & 0.047 $\pm$ 0.007 & \textbf{0.388} $\pm$ 0.222 \\
TwistNet - DRE & 0.066 $\pm$ 0.022 & 0.842 $\pm$ 0.136 & 0.328 $\pm$ 0.039 &\textbf{ 0.049} $\pm$ 0.014 & 0.797 $\pm$ 0.319 \\
TAG - DRE & 0.067 $\pm$ 0.085 & 0.723 $\pm$ 0.271 & 0.307 $\pm$ 0.104 & 0.096 $\pm$ 0.130 & 1.891 $\pm$ 3.075 \\
NAS-X & \textbf{0.092} $\pm$ 0.024 & 0.777 $\pm$ 0.123 & 0.359 $\pm$ 0.044 & \textbf{0.049} $\pm$ 0.009 & 0.500 $\pm$ 0.293 \\
NASMC & 0.015 $\pm$ 0.017 & 0.640 $\pm$ 0.326 & 0.329 $\pm$ 0.105 & 0.035 $\pm$ 0.021 & 1.728 $\pm$ 0.981 \\
\end{tabular}

\vspace{2em}

\includegraphics[width=0.95\linewidth]{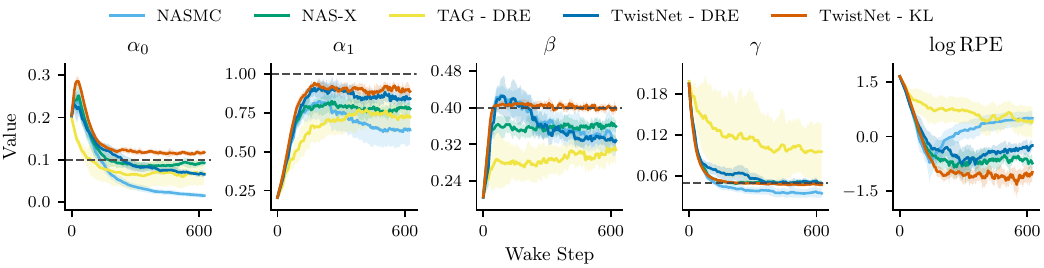}
\vspace{1em}

\captionof{figure}{ Evolution of the parameters and log relative parameter error ($\log\,$RPE) through wake steps updating the parameters $\theta$, for a SIRS model on a graph with 64 nodes.}
\label{fig:params-64}

\end{figure}
 
\paragraph{Twist model.} We use the same UTAE model for the context encoder of the twist and the score networks, with channel multiplicities of $[64, 64, 64, 128]$ during upsampling and $[128,64,32,32]$ during downsampling, and embedding dimension of 64 for self-attention. We let the context encoder depend on the set of weekly observations $y_{1:7}\in\{0,1\}^{64\times 64\times 1\times 7}$ and VIIRS covariates $x_{1:7}\in \R^{64\times 64\times 3\times 7}$, concatenated along the channel dimension. We also let the model depend on the scalar time $t$, by applying adaptive normalization layers \cite{perez2018film} to intermediate activations, following the implementation in \cite{peebles2023scalable}. 
The score network (used in NASMC and NAS-X) is identical, but it also takes as input the current one-hot encoded state $\mathbf{z}_t\in\{0,1\}^{64\times 64\times 3}$ by repeating it 7 times and concatenating it to the other inputs along channel dimensions.
For both the context encoder and the score we use an additional convolutional layer to map to the desired output channel dimension, corresponding to 3 for the score net and $3\times m$ for the TwistNet, where $m=256$. In the TwistNet, these embeddings are then passed through a two-layer MLP to produce the logarithm of twist values. We let the posterior initial distribution be a UTAE with channels $[32, 32, 64]$ for upsampling and $[64,32,32]$ for downsampling, conditioned on VIIRS channels and future observations. We also train an encoder conditioned only on VIIRS channels and the first observation.
NAS-X makes use of an additional model producing the twist value. We let this be identical to the score network, and perform mean pooling over the output to get a scalar value. This makes the number of parameters for NAS-X and TwistNet roughly equal: $2,845,961 =1,268,355\times2$ for the score and twist networks in NAS-X, and $2,808,420 = 1,709,664 + 789,505$ for the TwistNet, combining context encoder and aggregator.

\paragraph{Training.}
We follow the wake-sleep routine detailed in Algorithm~\ref{algo:wake_sleep_detailed}.
We start by training the $h^\psi$, the initial distribution, and the first-observation encoder using the Adam optimizer (learning rate $5\times 10^{-4}$, PyTorch defaults otherwise) with batch size $16$ for $150$ steps. We employ a Monte Carlo approximation of the time summation at each training step, and the time grid is built by sampling $\Delta_t$ uniformly from $\{0.02,\,0.05,\,0.1\}$. We found that including finer grids helps the twist model learn how to interpolate between observations where multiple transitions occur, which is typical in this dataset given the rapid spread of wildfires.
We then alternate blocks of wake updates for $\theta$ and sleep updates for $\psi$ for $G{=}15$ outer iterations. We amortize computation by performing 25 gradient updates per simulated batch. During wake steps, we use the Adam optimizer on $\theta$ with the same learning rate of $5\times 10^{-4}$, batch size of $8$, and $40$ updates per outer iteration. 
Each update uses tSMC with proposal $q^{\bar\theta,\psi}$, $M{=}5$ particles, systematic resampling at every step (ESS threshold $=1$), and $\Delta_t$ sampled uniformly from $\{0.05,\,0.1\}$. Because we resample at every step, we avoid weighting the wake loss across particles. Instead, we draw one path by importance resampling using the final normalized weights and evaluate the wake loss on that path.

\begin{figure*}[!t]
  \centering

  \begin{subfigure}[t]{0.48\textwidth}
    \includegraphics[width=\linewidth]{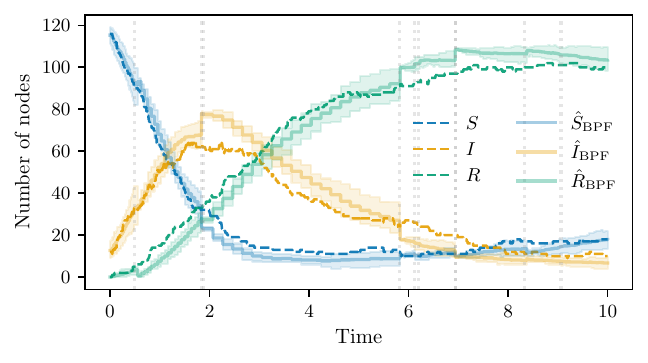}
  \end{subfigure}\hfill
  \begin{subfigure}[t]{0.48\textwidth}
    \includegraphics[width=\linewidth]{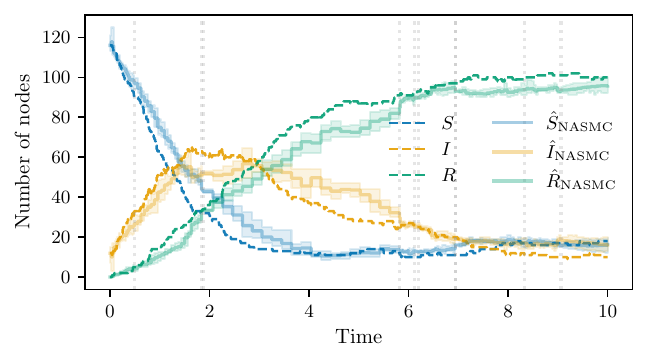}
  \end{subfigure}

  \medskip 

  \begin{subfigure}[t]{0.48\textwidth}
    \includegraphics[width=\linewidth]{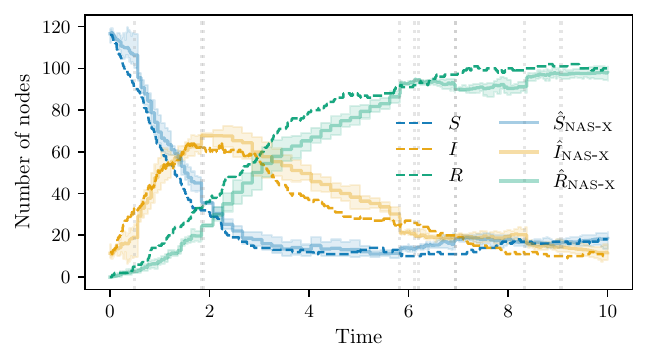}
  \end{subfigure}\hfill
  \begin{subfigure}[t]{0.48\textwidth}
    \includegraphics[width=\linewidth]{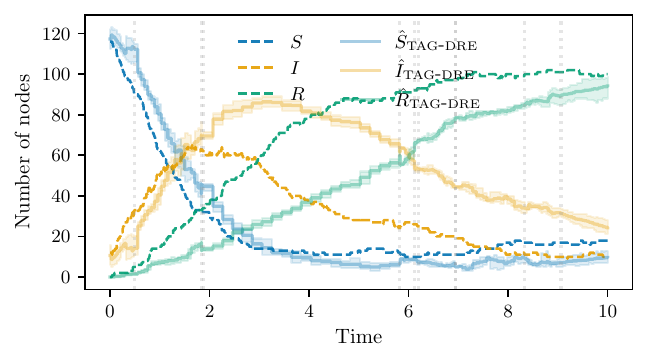}
  \end{subfigure}

  \medskip

  \begin{subfigure}[t]{0.48\textwidth}
    \includegraphics[width=\linewidth]{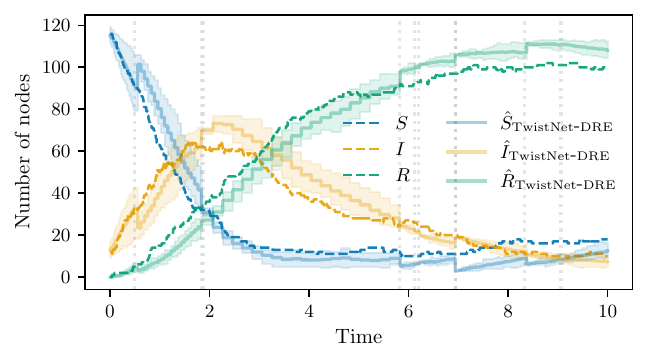}
  \end{subfigure}\hfill
  \begin{subfigure}[t]{0.48\textwidth}
    \includegraphics[width=\linewidth]{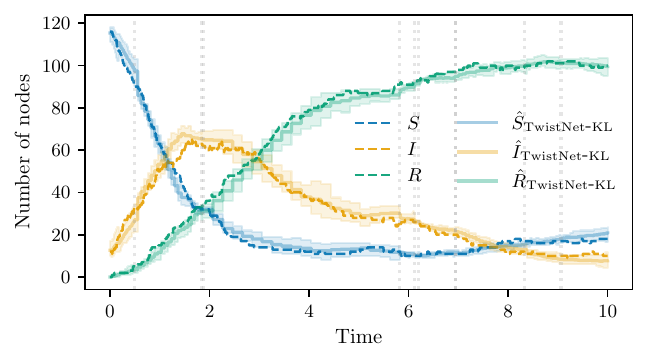}
  \end{subfigure}

  \caption{First example of latent trajectories, with counts of each state over the graph at each timestep. }
  \label{fig:example1}
\end{figure*}

\begin{figure*}[!t]
  \centering

  \begin{subfigure}[t]{0.48\textwidth}
    \includegraphics[width=\linewidth]{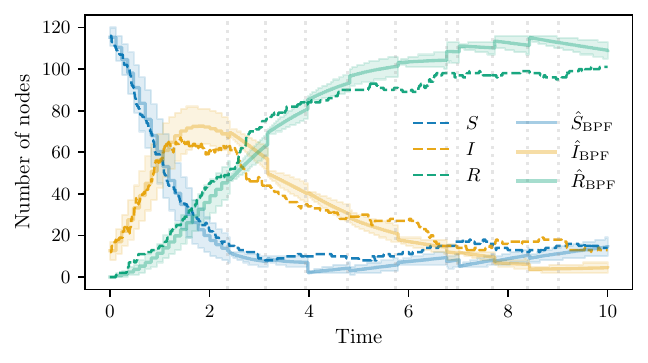}
  \end{subfigure}\hfill
  \begin{subfigure}[t]{0.48\textwidth}
    \includegraphics[width=\linewidth]{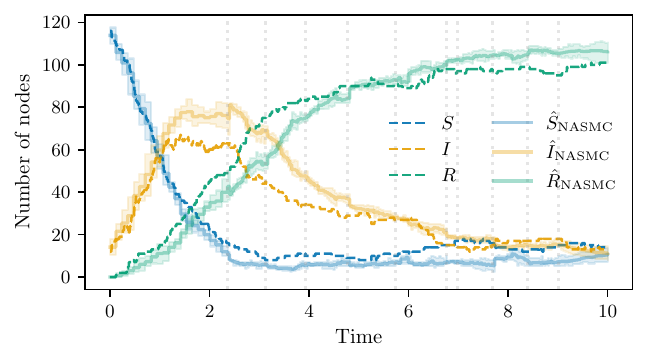}
  \end{subfigure}

  \medskip 
  
  \begin{subfigure}[t]{0.48\textwidth}
    \includegraphics[width=\linewidth]{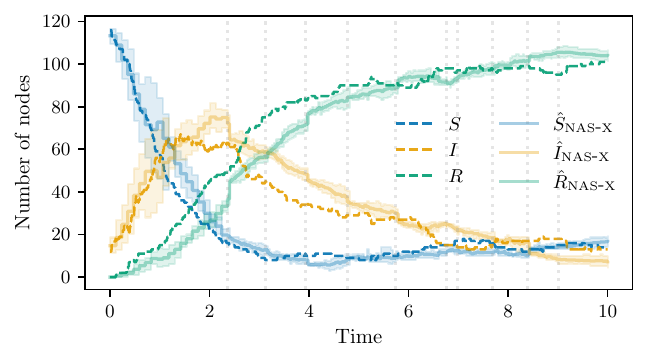}
  \end{subfigure}\hfill
  \begin{subfigure}[t]{0.48\textwidth}
    \includegraphics[width=\linewidth]{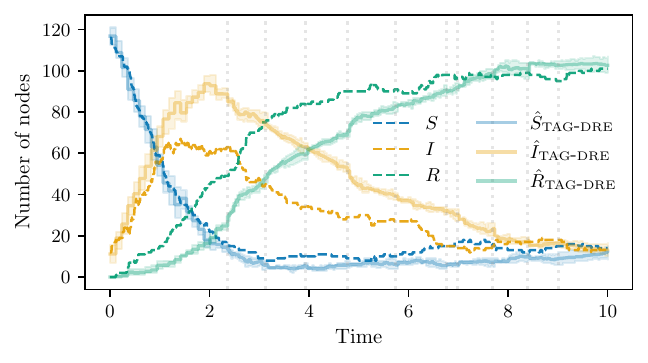}
  \end{subfigure}

  \medskip

  \begin{subfigure}[t]{0.48\textwidth}
    \includegraphics[width=\linewidth]{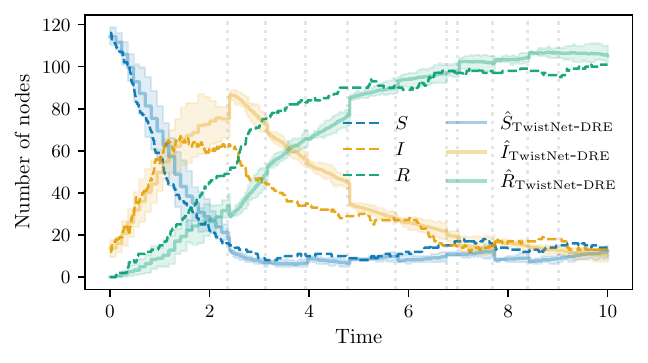}
  \end{subfigure}\hfill
  \begin{subfigure}[t]{0.48\textwidth}
    \includegraphics[width=\linewidth]{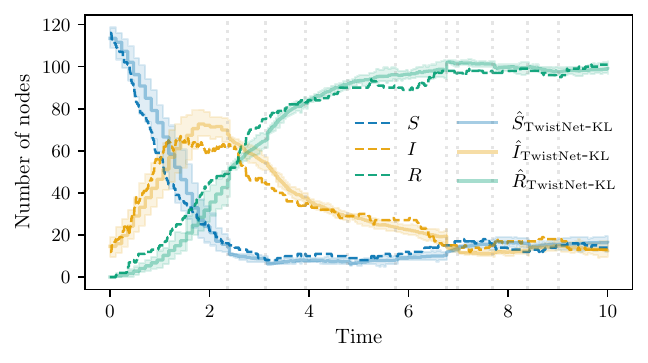}
  \end{subfigure}

  \caption{Second example of latent trajectories, with counts of each state over the graph at each timestep. }
  \label{fig:example2}
\end{figure*}

\paragraph{Evaluation.}
We consider two tasks: 
\begin{itemize}
    \item \emph{Reconstruction}: we generate trajectories conditioning on the full observation sequence $y_{1:7}$, and evaluate how well we recover the active fire areas. We run tSMC (Algorithm~\ref{algo:tsmc_short}) (or SMC for NASMC) with $S{=}16$ particles, $\Delta_t{=}0.02$, and resample at every step (uniform weights at the daily grid). We found that the performance of every method could be improved in this task by introducing a temperature parameter $\alpha$, and modifying both the twist and the score by raising them to the power of $\alpha$. To set the temperature, we grid-search $\alpha\in\{0.05,\,0.10,\,0.25,\,0.50,\,1.0\}$ on the first three training batches and fix the best $\alpha^\star$ for each method for evaluation. 
    \item \emph{Prediction}: after conditioning on the first observation, we run the prior model for $S=16$ samples per set of observations and see how well the resulting empirical distribution can predict future trajectories.
\end{itemize}
For both of these tasks, the metric we use is the binary cross-entropy of the empirical distribution from the particles against the observed active-fire map, averaging over pixels and days per trajectory. We report mean $\pm 2$ standard errors across 26 test trajectories in Table~\ref{tab:wildfire-results}. 

\begin{figure}[ht]
    \centering
    \includegraphics[width=0.95\linewidth]{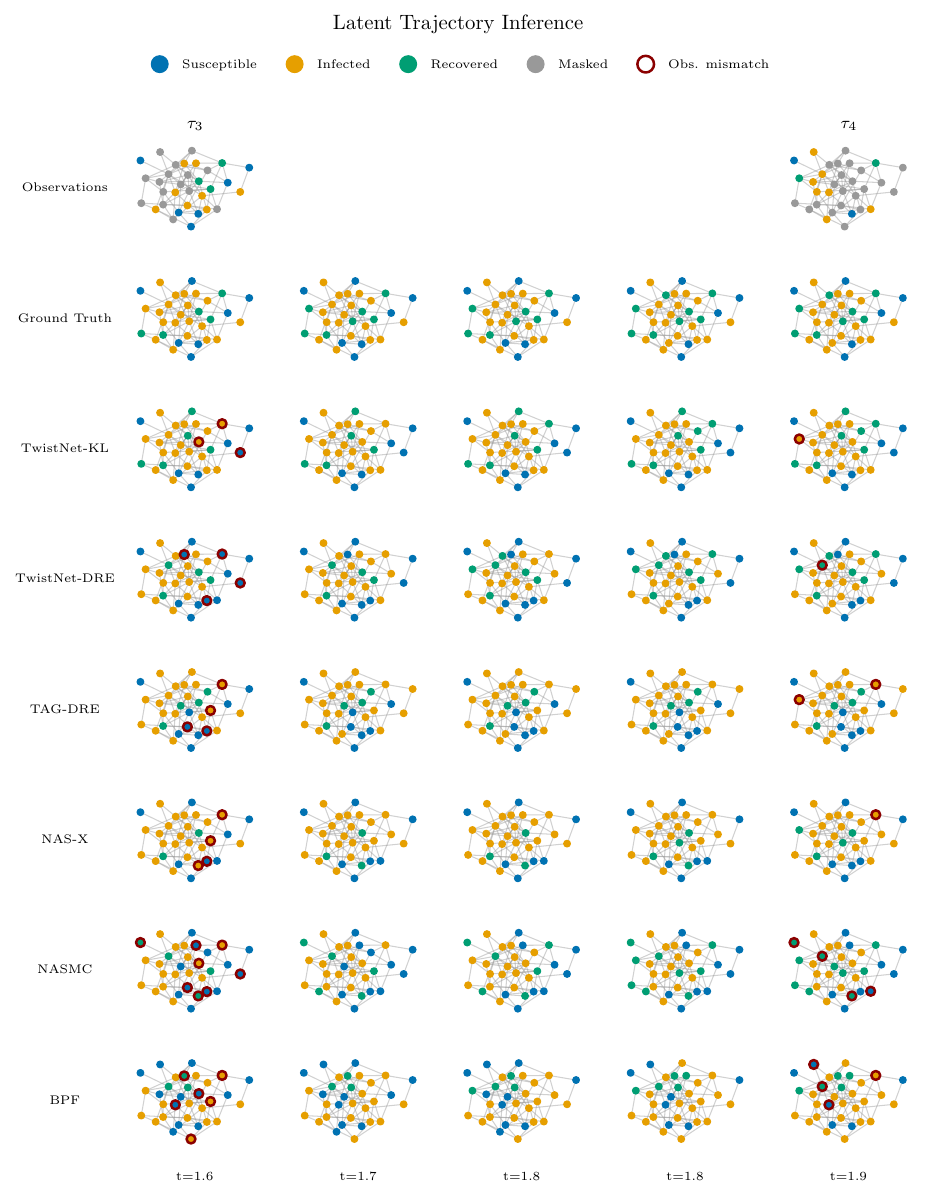}
    \caption{Samples in between observations at times $\tau_3=1.6$ and $\tau_4=1.9$, for a graph with 32 nodes. Mismatches between unmasked observations and samples are highlighted in red.}
    \label{fig:sirs32}
\end{figure}

\begin{figure}[ht]
    \centering
    \includegraphics[width=0.9\linewidth]{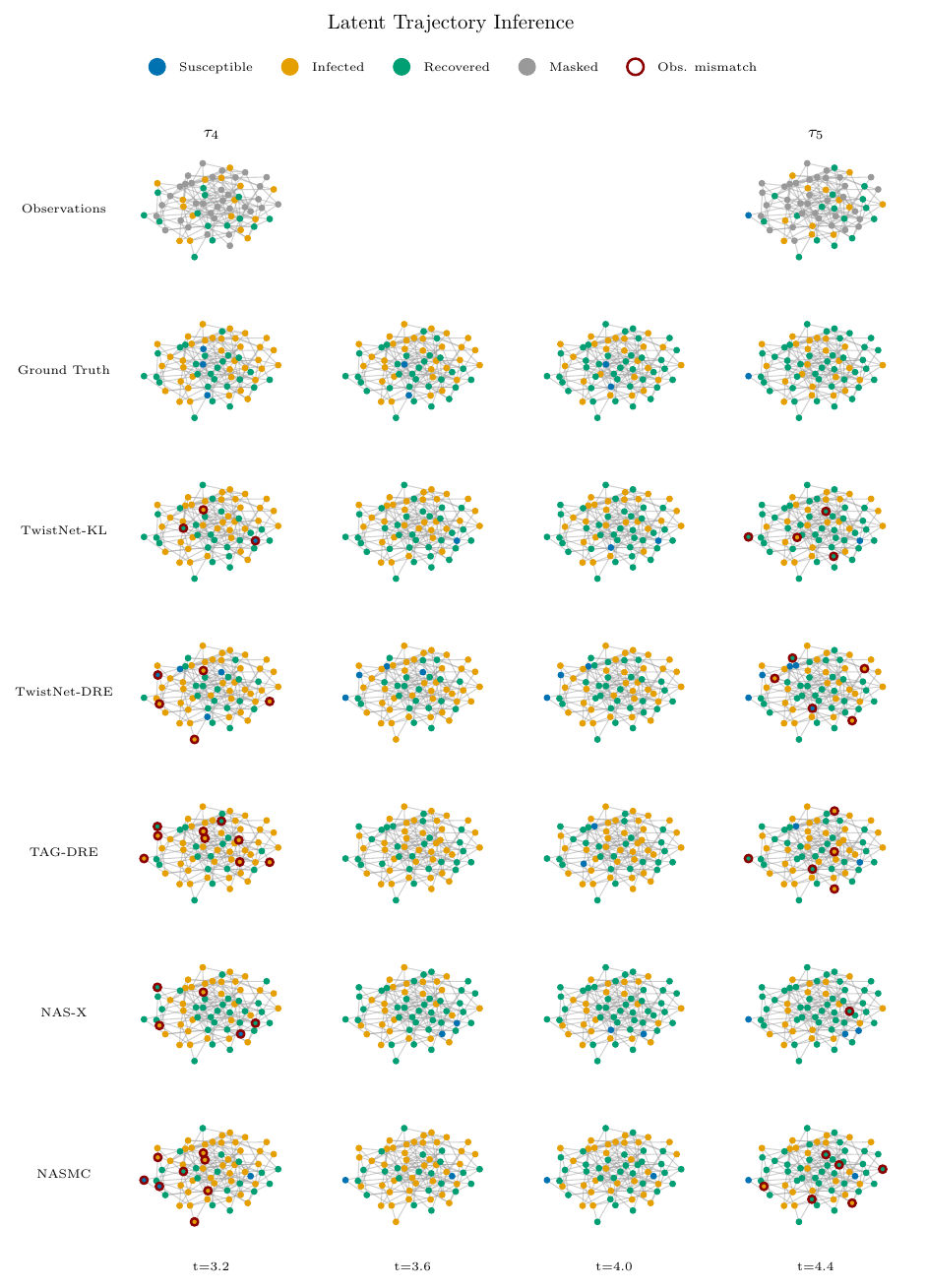}
    \caption{Samples in between observations at times $\tau_4=3.2$ and $\tau_5=4.4$, for a graph with 64 nodes. Mismatches between unmasked observations and samples are highlighted in red.}
    \label{fig:sirs64}
\end{figure}
\section{Additional Results}\label{appdx:add_results}
\subsection{SIRS model}\label{appdx:add_sirs}
We replicate the parameter learning experiment for graphs with $d=64$ nodes using 10 random seeds, and display the results in Table \ref{tab:final-param-estimates-64} and the values at each step in Figure \ref{fig:params-64}. By comparing it with the results for 32 nodes, we can notice that the gap between the TwistNet with KL loss and the other methods widens as dimensionality increases, in line with the experiment on latent trajectory inference.

In Figures \ref{fig:example1} and \ref{fig:example2}, we show a summary of generated latent trajectories for each method for different graphs with $d=128$ nodes, displaying a count of nodes in each state. In line with the evaluation in Figures \ref{fig:nll_vs_dim} and \ref{fig:brier}, the TwistNet-KL method seems to display much better performance than the alternatives in terms of closeness to the ground truth. In Figure \ref{fig:sirs32} and Figure \ref{fig:sirs64} we display a segment between two observations for graphs with 32 and 64 nodes respectively, and highlight the mismatch of the latent trajectories with the (non-masked) observations.

While twisted SMC methods shows huge improvements in performance over traditional schemes for high-dimensional problems, this comes at a cost: the need to perform a forward pass of the neural network at each timestep. This significantly increases the runtime of this family of methods, and this is possibly the biggest limitation of these methods when applied to large datasets.

\subsection{Wildfires trajectories}\label{appdx:add_wf}
In Figure \ref{fig:wf1} we show the VIIRS channels and empirical prior marginals (for $S=16$ samples and $\Delta_t=0.02$) for the trajectories displayed analyzed in Figure \ref{fig:wildfire}. Prior samples are obtained after encoding the first observation. We note that, other than TwistNet with a KL loss, none of the baselines seems to have learned a meaningful model for the prior dynamics, despite their losses having converged.

Additional examples of ground truth, inputs, prior predictions and posterior reconstructions on the test set are displayed in Figure \ref{fig:wf2} and Figure \ref{fig:wf3}, all using $S=16$ samples and $\Delta_t=0.02$. For the prior we display the empirical distribution of active fire, whereas for the posterior we display a single (non-cherry-picked) sample. A particularly challenging scenario is displayed in Figure \ref{fig:wf3} on the right: the fire has a single burning pixel at the start but develops into a large fire at days 6 and 7, and VIIRS channels are missing at two timesteps. TwistNet with a KL loss tries to smoothly interpolate, slightly accelerating between the fifth and the sixth day. All the other methods fail, and none of the predictions from the prior model seems to capture the possibility of a rapid spread. 

\begin{figure}[!t]
    \centering
    \includegraphics[width=0.99\linewidth]{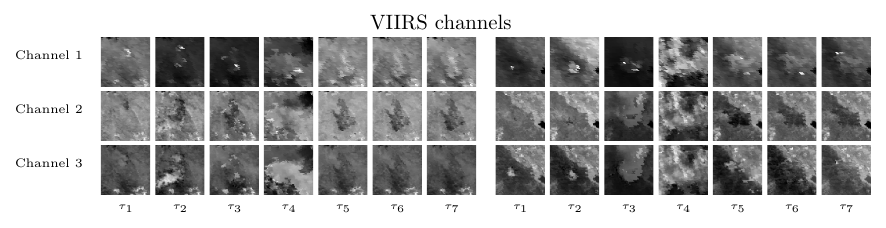}
    \includegraphics[width=0.99\linewidth]{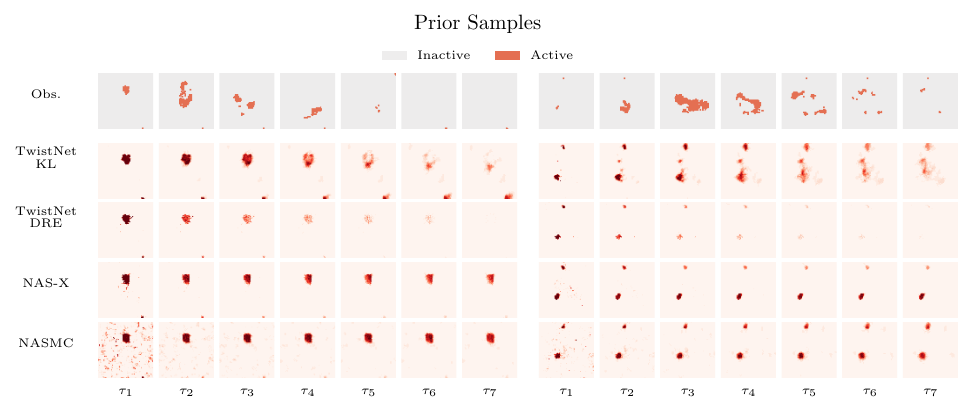}
    \caption{VIIRS covariates and empirical distribution of active fires obtained from samples of the prior dynamics, for the examples presented in Figure \ref{fig:wildfire}.}
    \label{fig:wf1}
\end{figure}

\begin{figure}[!t]
    \centering
    \includegraphics[width=0.99\linewidth]{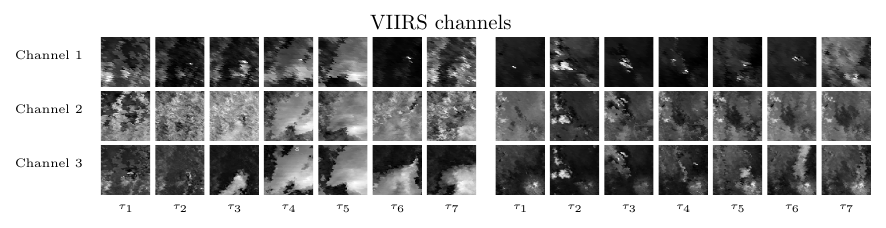}
    \includegraphics[width=0.99\linewidth]{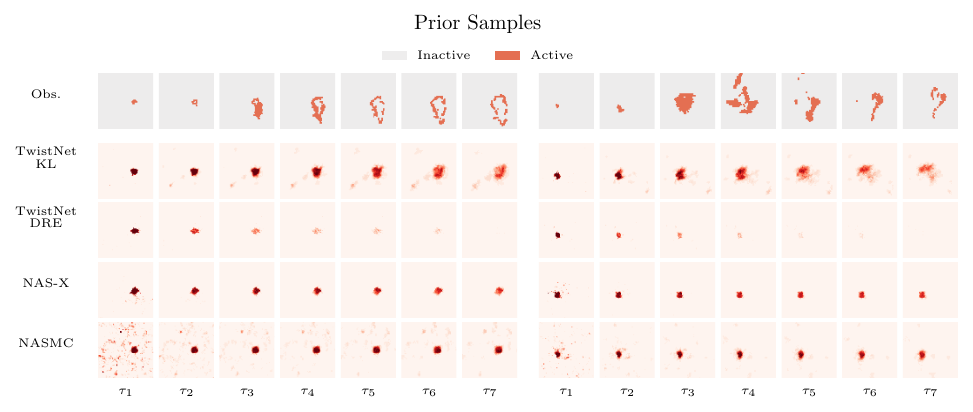}
    \includegraphics[width=0.99\linewidth]{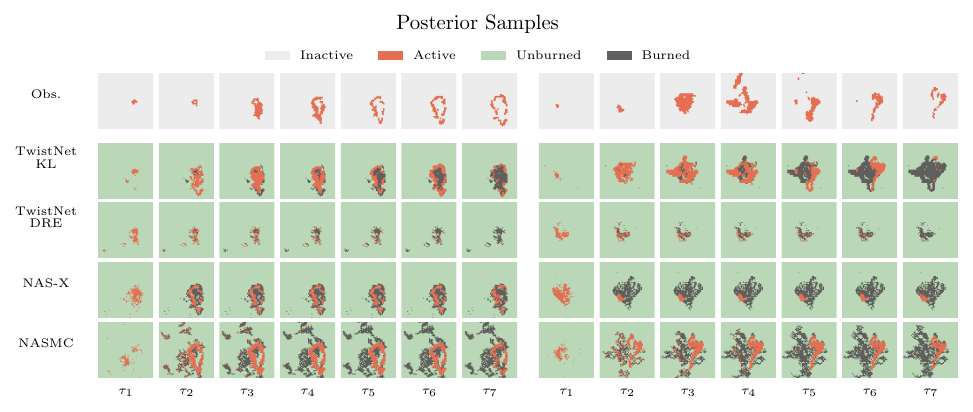}
    \caption{Third and fourth examples of wildfire trajectories.}
    \label{fig:wf2}
\end{figure}

\begin{figure}[!t]
    \centering
    \includegraphics[width=0.99\linewidth]{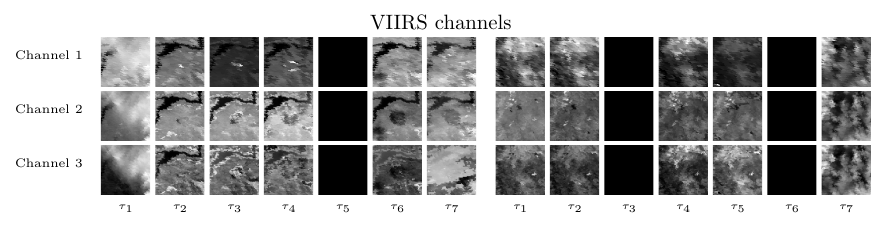}
    \includegraphics[width=0.99\linewidth]{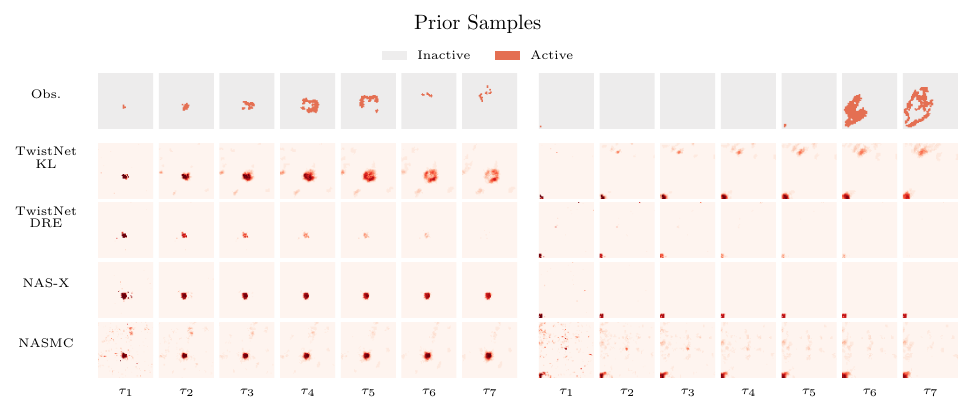}
    \includegraphics[width=0.99\linewidth]{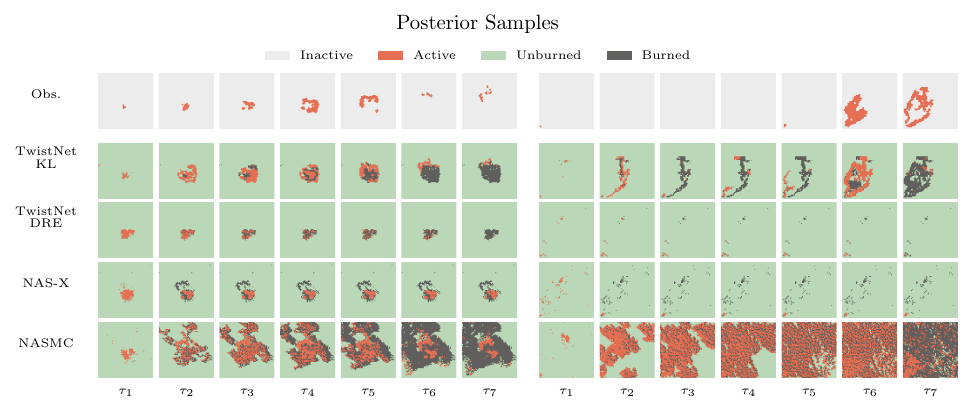}
    \caption{Fifth and sixth examples of wildfire trajectories.}
    \label{fig:wf3}
\end{figure}

\vfill

\end{document}